\documentclass[12pt]{article}

\usepackage{latexsym}
\usepackage{setspace}
\usepackage{amsmath,amsfonts,amsxtra,amsthm,amssymb,euscript}
\usepackage{url}
\usepackage{graphicx}
\usepackage{subfigure}
\usepackage{wrapfig}
\usepackage{titlesec}
\usepackage{fancybox,color,colortbl}
\usepackage[footnotesize]{caption}
\usepackage{cite}
\usepackage{epsfig} 
\usepackage{algorithm}
\usepackage{algorithmic}
\usepackage{xspace}
\usepackage{tweaklist}
\usepackage{mathrsfs}
\usepackage{pstricks,psfrag}
\usepackage{multirow}

\usepackage[top=1.0in, right=1.0in, left=1.0in, bottom=1.0in, centering]{geometry}

\def\rrr#1\\{\par
\medskip\hbox{\vbox{\parindent=2em\hsize=6.12in
\hangindent=4em\hangafter=1#1}}}

\newtheorem{theorem}{Theorem}[section]

\newtheorem{corollary}[theorem]{Corollary}
\newtheorem{definition}[theorem]{Definition}

\newtheorem{problem}[theorem]{Problem}

\titleformat{\section}{\Large\bfseries}{\thesection}{1em}{}
\titlespacing{\section}{3pt}{12pt}{10pt}
\titlespacing{\paragraph}{0pt}{6pt}{0pt}

\renewcommand{\Re}{\mathbb{R}}

\def\marhes{{\sc Marhes}\xspace}

\newcommand{\ie}{{\it i.e.},\xspace}
\newcommand{\eg}{{\it e.g.},\xspace}

\newcommand\bs{\begin{singlespace}}
\newcommand\es{\end{singlespace}}

\newcommand\oprocendsymbol{\hbox{$\bullet$}}
\newcommand\oprocend{\relax\ifmmode\else\unskip\hfill\fi\oprocendsymbol}

\newcommand{\footnoteremember}[2]{
\footnote{#2}
\newcounter{#1}
\setcounter{#1}{\value{footnote}}
}
\newcommand{\footnoterecall}[1]{
\footnotemark[\value{#1}]
}

\begin{document}
\vspace{-0.575in}
	\title{\bf Exploiting Heterogeneous Robotic Systems in Cooperative Missions }
	\author{Nicola Bezzo$^{1}$, Joshua P. Hecker$^{2}$, Karl Stolleis$^{2}$, \\ Melanie E. Moses$^{2}$, and Rafael Fierro$^{3}$\\
			$^{1}${\sc Precise} Center, School of Engineering \& Applied Science\\
			University of Pennsylvania, Philadelphia, PA, 19104, USA \\
			{\small {\tt nicbezzo@seas.upenn.edu}}\\
			$^{2}${\sc Biological Computation} Laboratory, Department of Computer Science\\
			University of New Mexico, Albuquerque, NM 87131, USA \\
			{\small {\tt \{jhecker,stolleis,melaniem\}@cs.unm.edu}}\\
			$^{3}${\sc Marhes} Laboratory, Department of Electrical \& Computer Engineering\\
			University of New Mexico, Albuquerque, NM 87131, USA \\
			{\small {\tt rfierro@ece.unm.edu}}
	}
	\date{}
\maketitle

\begin{abstract}
In this paper we consider the problem of coordinating robotic systems with different kinematics, sensing and vision capabilities to achieve certain mission goals. An approach that makes use of a heterogeneous team of agents has several advantages when cost, integration of capabilities, or large search areas need to
be considered. A heterogeneous team allows for the robots to become ``specialized'', accomplish sub-goals more effectively, and thus increase the overall mission efficiency. Two main scenarios are considered in this work. In the first case study we exploit mobility to implement a power control algorithm that increases the {\em Signal to Interference plus Noise Ratio} (SINR) among certain members of the network. We create realistic sensing fields and manipulation by using the geometric properties of the sensor field-of-view and the manipulability metric, respectively. The control strategy for each agent of the heterogeneous system is governed by an artificial physics law that considers the different kinematics of the agents and the environment, in a decentralized fashion. Through simulation results we show that the network is able to stay connected at all times and covers the environment well. The second scenario studied in this paper is the biologically-inspired coordination of heterogeneous physical robotic systems. A team of ground rovers, designed to emulate desert seed--harvester ants, explore an experimental area using behaviors fine-tuned in simulation by a genetic algorithm. Our robots coordinate with a base station and collect clusters of resources scattered within the experimental space. We demonstrate experimentally that through coordination with an aerial vehicle, our ant-like ground robots are able to collect resources two times faster than without the use of heterogeneous coordination.

\end{abstract}

\section{Introduction}
\label{sec:1}
In recent years we have witnessed an increase in the use of mobile robots for different applications spanning from military to civilian operations. Search and rescue missions, disaster relief operations, and surveillance are just few examples of scenarios where the use of autonomous and intelligent robotic systems is preferred over the use of human first responders. In such operations wireless communication needs to be reliable over the robotic network to maneuver the unmanned vehicles and transmit information. We are interested in heterogeneous robotic systems with agents having different kinematics, sensing behaviors, and functionalities. 
For instance we will consider quadrotor aerial vehicles, that can be approximated as holonomic agents, interacting with ground robots (\eg non-holonomic, car-like agents), both with different communication ranges and sensing and manipulation patterns. Fig. \ref{fig:img1} shows an example of heterogeneous systems with quadrotors cooperating with ground vehicles and crawling agents, all systems acting as communication and sensing relays. 

Within this paper the contribution to the current research on distributed robotic systems is fourfold: {\em i)} we consider heterogeneous robotic systems with different dynamics and realistic communication analysis, sensing geometries, and manipulation constraints, in a decentralized fashion; {\em ii)} we build a power control algorithm for communication purposes to improve the SINR among certain members of the network; {\em iii)} we extend our previous work \cite{andres2011dars, bezzodars2012} considering a more general and realistic scenario while showing that the heterogeneous system stays connected all the time; and {\em iv)} we consider a biologically-inspired scenario to coordinate heterogeneous robots to harvest resources in an unmapped area.
Throughout this work we integrate together several tools for coordination and control of distributed heterogeneous robotic systems.

\begin{figure}[ht!]
\begin{center}
\subfigure[] {\includegraphics[width=0.315\textwidth]{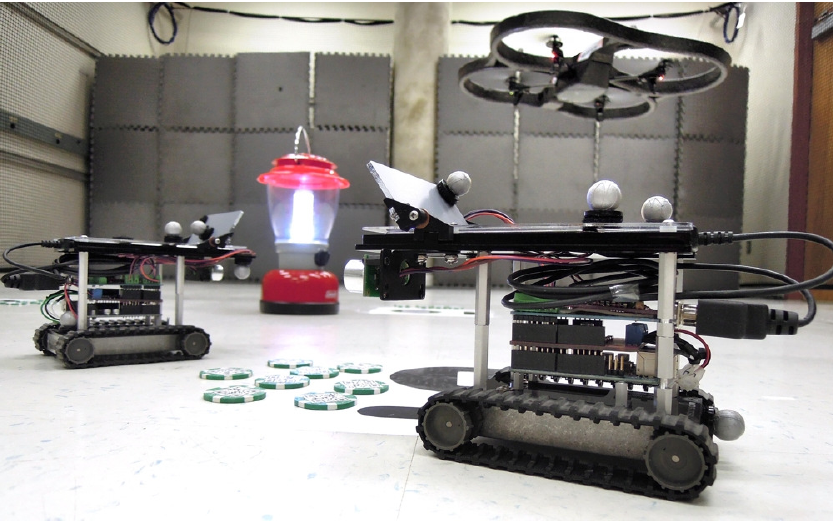}}
\subfigure[] {\includegraphics[width=0.326\textwidth]{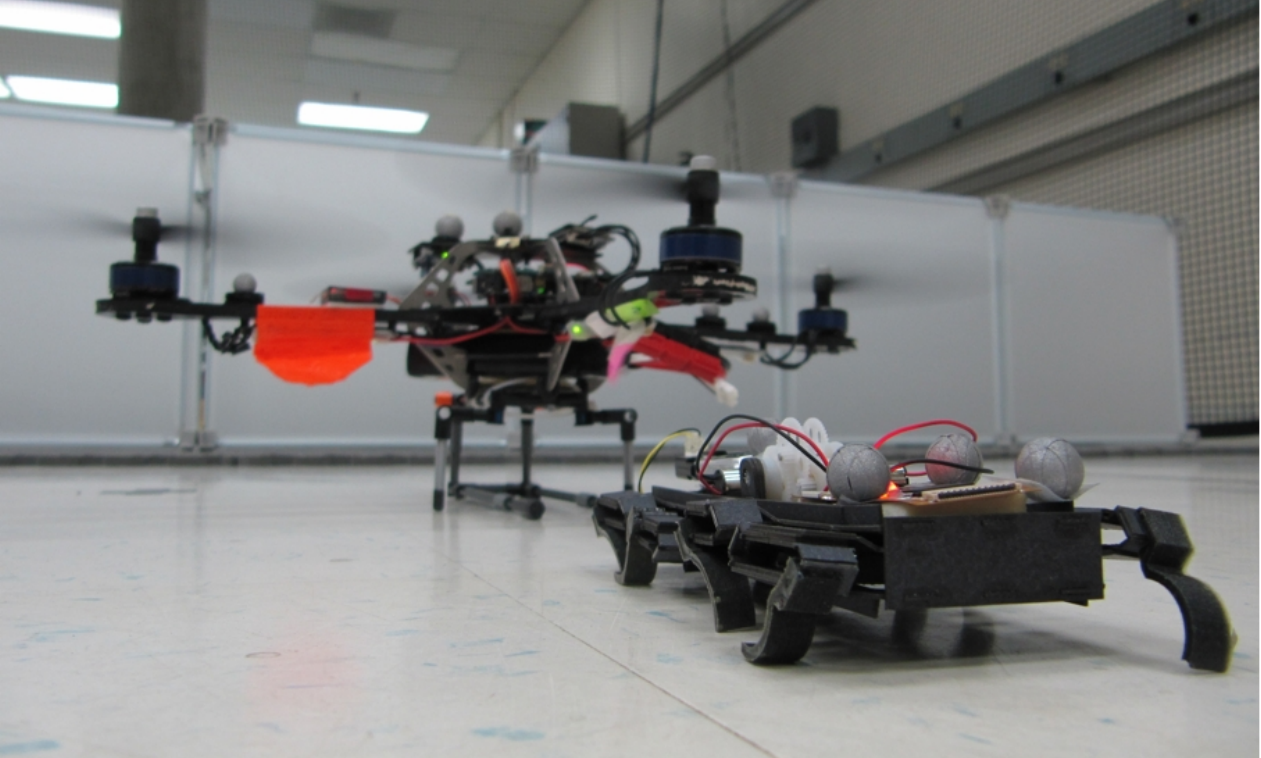}}
\subfigure[] {\includegraphics[width=0.3235\textwidth]{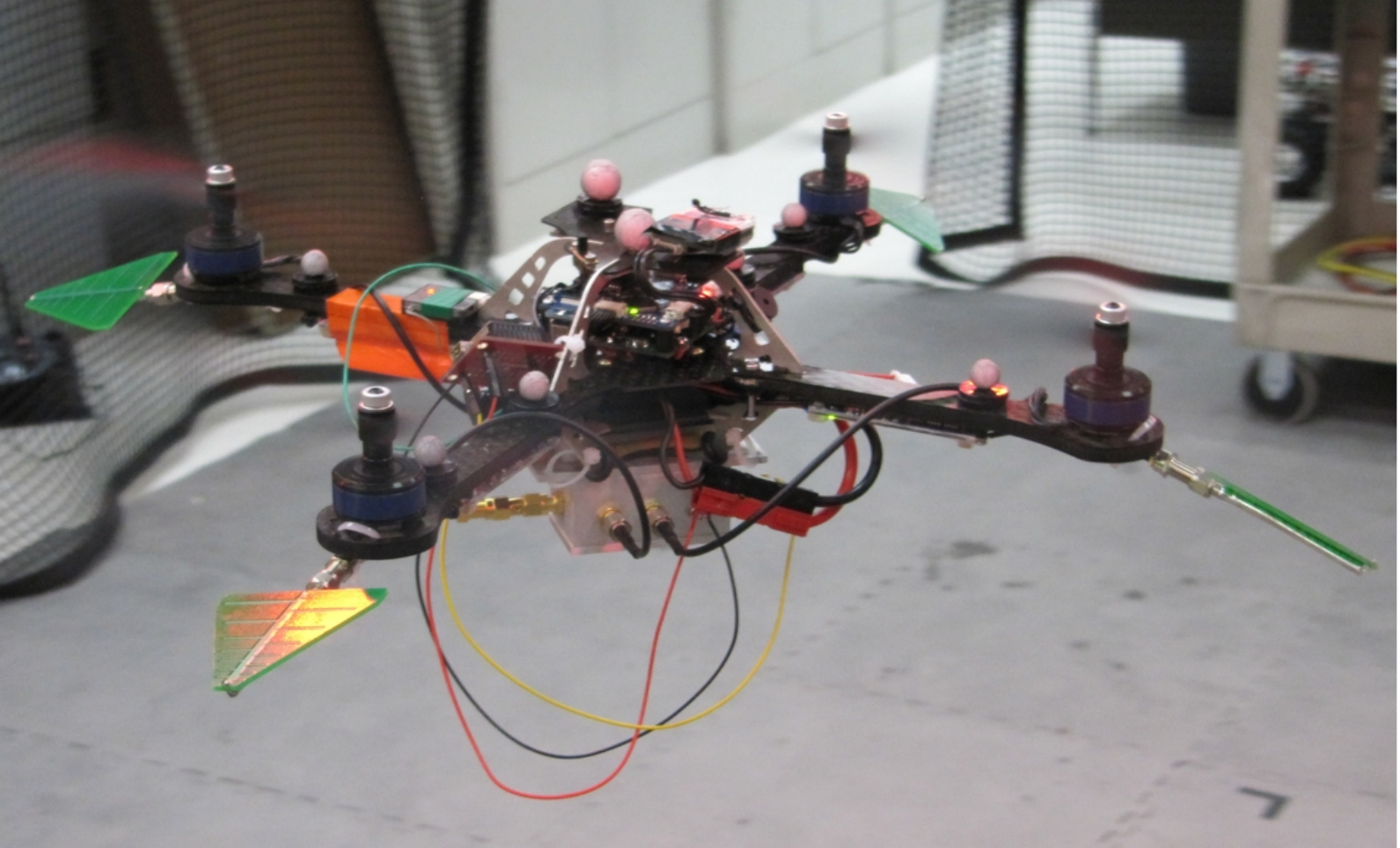}}
\end{center}
\caption{\label{fig:img1}(a) A group of iAnt robots \cite{hecker2012formica} (Biological Computation Laboratory) cooperating with an AR.Drone \cite{Parrot}; (b) Deployment of the OctoRoACH crawling robot \cite{pullin2012dynamic} using a quadrotor \cite{AscTec_Hummingbird}  ({\sc MARHES} Laboratory); and (c) An example of aerial mobile relay with four directional antennas.}
\end{figure}

\subsection{Related Work}\label{sec:Relatedwork}
Heterogeneity in robotic applications is attracting recent attention because of the challenges created by multi-agent systems having different kinematics, sensing, and manipulation capabilities.
Authors in \cite{abbas2011distribution} consider formally a heterogeneous system and analyze its properties based on graph coloring techniques to assign colors to different types of agents. Similarly to the work presented in this paper, authors in \cite{di2011decentralized} use agents with different dynamics and capabilities to execute multiple missions in a decentralized fashion considering task sequencing and a consensus-based technique. In \cite{kumar2010segregation} the authors introduce control laws based on differential potential for aggregation and segregation of biologically inspired heterogeneous agents. 

In missions involving multi-agent systems, it is necessary to consider wireless communication to maintain network connectivity at all times. The robotics and control community are very active in investigating the integration of communication in robotics applications, because the uncertainties found in wireless channels can compromise the performance of the entire multi-agent system. For instance, authors in \cite{hollinger2011autonomous} propose a modified Traveling Salesperson Problem to navigate an underwater vehicle in a sensor field, using a realistic model that considers acoustic communication fading effects. In \cite{hovakimyansingular} a Rician fading model for the communication channel is utilized in a pursuit-evasion game with two mobile agents moving in a cluttered environment. In \cite{zavlanosdistributed} the authors optimize routing probabilities to ensure desired communication rates while using a distributed hybrid approach.
In \cite{bezzo2011disjunctive} we tether a chain of mobile routers to keep line-of-sight communication between a base station and a user  that moves in a concave environment. Authors in \cite{vieira2011mitigating} show extensive experimental results to optimize the communication throughput by making small variations in the positions of agents in the environment. 

Similarly to the work presented in this paper, the authors in \cite{how2011} present a multi-agent system with interaction between aerial and ground vehicles based on task assignment for complex missions. From a graph-theoretical point of view \cite{zavlanos2011graph} surveys graph connectivity in mobile robot swarms, discussing different approaches and algorithms to maintain and optimize connectivity among mobile robot networks. Still from a connectivity perspective, authors in \cite{cezayirlinavigation} use position information to maintain a network of mobile robots connected without the need of communication among the agents.

The communication community has been investigating cognitive radio antennas to improve the SINR in cellular networks \cite{ao2012connectivity}. These devices change the transmission and reception parameters to improve the overall communication quality. One of the most common ways to improve the SINR is to use {\em Power Control} (PC) algorithms in which all wireless devices adjust their power level to reach a desired SINR threshold \cite{zhou2012reinforcement, liu2012optimal, goldsmith2005wireless}. In the work presented in this paper we consider a similar PC approach, but we exploit the mobility of the mobile agents to change the received power at a certain location and reach a desired SINR.

Finally from a sensing point of view, authors in \cite{schwager2011eyes} present an optimization framework to maneuver aerial vehicles equipped with cameras to perceive a certain area based on field of view properties.

The remainder of this paper is organized as follows. In Section \ref{sec:2}, we define the heterogeneous system and formulate the connectivity problem considering relay, sensor, and manipulator agents. In Section \ref{sec:3}, we present the first case study in which aerial vehicles, sensor agents, and mobile manipulators cover a cluttered area in search of a fixed target protected by an opposing player. We analyze a power control method to improve the SINR over the network and we consider sensing and manipulation constraints. In Section \ref{sec:4} we introduce a biologically-inspired technique to coordinate groups of ant-inspired robots together with an aerial drone to collect resources without central control. Simulation and experimental results are presented to validate the proposed strategies. Finally, we draw conclusions and outline future work in Section \ref{sec:5} .

\section{Heterogeneous Connected Robotic System}\label{sec:2}
In this section we give a formal definition of  a heterogeneous robotic network followed by the problem formulation and connectivity constraints used to create interactions among the hybrid network.
\begin{definition}(Heterogeneous System): A network of $\mathcal N$ robots is called {\em heterogeneous} if the members of the network are interconnected, act together toward a common objective and if the following conditions hold:
\begin{itemize}
\item{one or more agents in the network have different motion dynamics with respect to other agents in the system;}
\item{one or more agents in the network have different sensing/manipulation constraints or improved wireless communications abilities with respect to other agents in the systems, but all agents have at least some wireless communication capabilities.}
\end{itemize}\label{def:het}
\end{definition}
\subsection{Heterogeneous Network Topology}
While the theoretical analysis presented here can be generalized for any type of network, we decide to focus on heterogeneous groups made of three types of mobile agents:
\begin{itemize}
\item{$\mathcal N_c$ communication relays with communication range $R_c>0$ and holonomic kinematics (\ie aerial vehicles like quadrotors). The set of relays is denoted by $\mathcal A_c$.}
\item{$\mathcal N_s$ mobile sensors with communication range $0<R_s<R_c$ and non-holonomic kinematics given by the bicycle model
\begin{equation}\label{eq:nh}
\mathbf{u}_k=\left[ \begin{array}{c} \dot{x} \\ \dot{y} \\ \dot{\theta} \\ \dot{\gamma} \end{array} \right] = \begin{bmatrix} \cos(\theta) & 0\;{} \\ \sin(\theta) & 0\;{} \\ \frac{1}{L} \tan(\gamma) & 0\;{} \\ 0 & 1\;{} \end{bmatrix} \left[ \begin{array}{c} v \\ w \end{array} \right], \qquad w = {\lambda}_s({\gamma}_d-\gamma)
\end{equation}
where $L$ is the distance between the front and rear axles, $v$ is the velocity, $w$ is the steering command described by a 1\textsuperscript{st} order linear servo model, ${\lambda}_s$ is the servo gain, and ${\gamma}_d$ is the desired steering angle.
The set of all mobile sensors is denoted by $\mathcal A_s$}.
\item{$\mathcal N_m$ manipulator agents with communication range $0<R_m(=R_s)<R_c$ and non-holonomic kinematics \eqref{eq:nh}.The set of all manipulators is denoted by $\mathcal A_m$.} 
\end{itemize}

The specific problem we are interested in this paper is the following:
\begin{problem}\label{prob:pr}
{\bf Deployment of Heterogeneous Robotic Networks:} Given a heterogeneous robotic network of $\mathcal N$ agents partitioned by $\mathcal N_c$, $\mathcal N_s$, and $\mathcal N_m$, find a set of feasible policies ${\mathbf u}_i \in \mathcal U$ for each agent such that the workspace of interest $\mathcal W$ is well covered, the network is always connected to a fixed base station $b$, and it is possible to reach and manipulate a target  $\mathcal D$ protected by an adversarial opponent $\mathcal T$ having unknown dynamics $\mathbf u_{\mathcal T}$. %
\end{problem}
The adversarial opponent attempts to capture the manipulator agents while the mobile sensors try to pursue and capture the adversary, if detected.

Each agent in the group has some sensing capabilities that are explored in detail in the following sections. For now we will focus on the connectivity problem and formulate an algorithm to expand the network and cover a specific environment.

%
\subsection{Connectivity Constraints}
Following our previous work \cite{andres2011dars} we build a connectivity algorithm by taking advantage of the communication properties of the heterogeneous network. Specifically we formulate connectivity constraints to expand the input set (accelerations, velocities, and in turn positions) the agents can choose from, while still guaranteeing connectivity at all times. 

We define that a relay agent $i$ can communicate with another relay agent $j$ if and only if $j\in \mathcal B_c^i$ with $\mathcal B_c^i=\mathcal B(\mathbf x_i, R_c)$ the ball centered in $i$ of radius $R_c$. A mobile sensor $k$ (or equivalently a mobile manipulator $q$) can communicate with a relay $i$ if and only if $k (\textrm{or} \,q ) \in \mathcal B_s^i=\mathcal B(\mathbf x_i, R_s)$.
However a relay agent $i$ can communicate with a mobile sensor $k$ (or equivalently a mobile manipulator $q$) if $k(\textrm{or} \,q )\in \mathcal B_c^i=\mathcal B(\mathbf x_i, R_c)$. Therefore, by exploiting this last constraint we can expand the sensor and manipulator agents in the environment relaxing continuous bidirectional communication constraints and thus explore a larger area of the workspace. These agents return within range of bidirectional communication with the relay when they have information relevant to the entire network.


At the beginning of a mission we consider a connected graph having the nodes placed in random positions. Also we create the following initial conditions
\begin{eqnarray}\label{eq:conn}
\begin{split}
\forall i \in \mathcal A_c , \exists \, j \in \mathcal A_c, i\neq j \;{\textrm {s.t.}}\; i\in \mathcal B (\mathbf x_j, R_c) \\
\forall k\in \{\mathcal A_s, \mathcal A_m\}, \exists \, j \in \mathcal A_c \;{\textrm {s.t.}}\; k\in \mathcal B (\mathbf x_j, R_c)
\end{split}
\end{eqnarray}

In order to have a uniform graph and maximize the coverage of a space, while maintaining connectivity, the connections between the agents of the heterogeneous system are biased based on the geometry of the communication radii. 
Since the sensor and manipulator agents have limited communication capabilities, the main idea is to have the communication relays connect to each other and expand the entire network in the environment. We consider that each communication relay is equipped with a high performance rf device that offers a large range and bandwidth to handle the communication with multiple nodes. Hence, each sensor and manipulator will be connected directly to a specific communication relay based on the minimum euclidean distance to the closest relay.
The {\em sensor/relay} and {\em manipulator/relay} assignments are built based on a local consensus algorithm described in Algorithm \ref{alg:Alg}.

\begin{algorithm}[h!]                      
\caption{Heterogeneous Local Consensus Algorithm}          
\label{alg:Alg}                           
\begin{algorithmic}                    
\WHILE{$t < t_{\text{final}}$}
\FOR{$i = 1,\ldots,\mathcal N_c$} \STATE Calculate the round down average number of sensor agents in the neighborhood of $i$\\
 $\underline{\tau}_i=\frac{n_i+\sum_{\{j\in \mathcal C_c^i | n_j\leq n_i\}} n_j}{\hat{\mathcal N}_j^i +1}$
\FOR{$j = 1,\ldots,\hat{\mathcal N}_j^i$}
\IF{$\exists \, k \in \hat{\mathcal A}_j^i\; \textrm{s.t.} \; n_i=n_j \forall j\in \{ \hat{\mathcal A}_j^i \setminus k\}$ and $n_k\leq (n_i -2)$}\STATE $\tilde{n}_k=n_k +\frac{(n_k+n_i)}{2}$ \STATE$\tilde n_i=n_i -\frac{(n_k+n_i)}{2}$ \ELSE
\STATE $\tau_j^i=\underline{\tau}_i-n_j$
\STATE $\tilde{n}_j=n_j+\tau_j^i$
\STATE $\tilde {n}_i=n_i-\tau_j^i$
\FOR{$l = 1,\ldots, \tau_j^i$}
\IF{$\exists p \in(\mathcal C_s^i \cap \mathcal B_c^j) \;\textrm{s.t.}\; ||\mathbf x_p-\mathbf x_j||=\min||\mathbf x_q-\mathbf x_j || \forall q\in \mathcal C_s^i$} 
\STATE $p\in \mathcal C_s^j$ and $p \notin \mathcal C_s^i$
\ENDIF
\ENDFOR
\ENDIF
\ENDFOR
\ENDFOR
\RETURN $n_i=\tilde{n}_i$
\ENDWHILE

%
%
%
%
%
%
%
%
%
%
%
\end{algorithmic}
\end{algorithm}
Specifically in Algorithm \ref{alg:Alg}, $\mathcal C_c^i$ and $\mathcal C_s^i$ are the set of neighbor relays and mobile sensors connected to the $i^{\textrm{th}}$ communication relay, respectively. $n_i$ is the number of sensors connected to the $i^{\textrm{th}}$ relay and $\tilde {n}_i$ the updated number of sensors after running the algorithm. Finally $\hat{\mathcal A}_j^i$ is the set containing $\hat{\mathcal N}_j^i$ communication relays connected to $i$ with $n_j\leq n_i$. Note that Algorithm \ref{alg:Alg} applies also to the manipulator agents in which we will have to consider $\mathcal C_m^i$ mobile manipulators' neighbors connected to $i$. If the graph is connected, then we can guarantee the network will reach at least a local consensus that is given by the average number of sensors and manipulators connected to the relays in the neighborhood of the $i^{\textrm{th}}$ relay \cite{olfati2007consensus}.


\section{A Real World Inspired Coordination of Connected Heterogeneous Robotic Systems}\label{sec:3}

In this section we present a framework to coordinate heterogeneous networks of aerial and ground agents while considering realistic communication, sensing, and manipulation constraints. We demonstrate the applicability of the proposed algorithm throughout a pursuit and evasion simulation \cite{bezzodars2012}.

\subsection{Motion Constraints}
$\bullet$ {\em Relay agent}: Given \eqref{eq:conn}, $\forall$ sensors $k \in C_s^i$ (or manipulators $q \in C_m^i$), if $||\mathbf x_i - \mathbf x_{k(q)}||\leq R_{\epsilon}$, with $R_s(R_m)<R_{\epsilon}<R_c$, the motion of the $i^{\textrm {th}}$ communication relay follows the spring-mass interaction
\begin{equation}\label{eq:u}
\ddot {\mathbf{x}}_{i} = \mathbf{u}_i,  \nonumber
\end{equation}
\begin{equation}\label{eq:spring}
\mathbf{u}_i =\left[ \sum_{j\in\mathcal C_c^i} \kappa_{ij}\left( l_{ij}-R_{\epsilon} \right) \hat{\mathbf d}_{ij} \right]-\delta_{i}\dot {\mathbf{x}}_{i}- \nabla_{\mathbf x_{i}}\varsigma(\mathbf x_{i}),
\end{equation}
where $\mathbf{u}_{i} \in \mathcal U$ ($\mathcal U \in \Re^3$) is the control input, ${\mathbf{x}_{i}=(x_{i}, y_{i}, z_{i})^T} $ is the position vector of the $i^{\text{th}}$ relay relative to a fixed Euclidean frame, and $\dot{\mathbf{x}}_{i} $,  $\ddot{\mathbf{x}}_{i}$ denote the velocity and acceleration (control input), respectively. $\mathcal C_c^i $ is the set of neighbor relays connected to the $i^{\text{th}}$ relay. $\mathcal C_c^i$ is built using the Gabriel Graph rule \cite{bezzo2011connectivity} in which between any two nodes $i$ and $j$, we form a virtual spring if and only if there is no robot $k$ inside the circle of diameter $\overline{ij}$, \cite{bezzo2011connectivity}. $l_{ij}$ and $\hat{\mathbf d}_{ij}$ are the length and direction of force of the virtual spring between robot $i$ and $j$ while $\kappa_{ij}$ and $\delta_{i}$ are the spring constant and damping coefficient, respectively. Here, we assume $\kappa_{ij}=\kappa_{ji}$ and $\delta_{i} > 0$.
$\nabla_{\mathbf x_{i}}\varsigma(\mathbf x_{i})=A_{\varsigma}(\mathbf x_{i}-c_{\varsigma})$ is the gradient of $\frac{A_{\varsigma}}{2}\| \mathbf x_{i}-c_{\varsigma}\|^2$, a quadratic attractive potential function where $c_{\varsigma}$ is the center of the region where the target is located and it is known {\em a priori}. 
\begin{theorem}
A network of communication relays having switching dynamics depicted in \eqref{eq:spring} is guaranteed to eventually reach stability in which all agents converge to a null state.
\end{theorem}
\begin{proof}
The proof for this theorem can be formulated using Lyapunov theory and can be found  in our previous work \cite{bezzo2011connectivity}.
\end{proof}

If $\exists \, k\in \mathcal C_s^i$ such that $R_{\epsilon}<||\mathbf x_i - \mathbf x_k||< R_c$ then the relay node believes that the specific sensor agent $k$ is in pursuit mode; therefore it switches into a follower mode with dynamics
\begin{equation}
 \mathbf{u}_i=\alpha(\mathbf{x}_{k}-\mathbf{x}_{i})
 \end{equation}\label{eq:purc}
where $\alpha \in \Re^{+}$. 

$\bullet$ {\em Sensor agent}: We consider the following interaction
\begin{equation}
\mathbf{u}_k=\left\{
    \begin{array}{ll}
        \kappa_{ki}\left( l_{ki}-R_s \right) \hat{\mathbf d}_{ki} -\delta_{k}\dot {\mathbf{x}}_{k} & \textrm{if $k \in \mathcal B (\mathbf x_i, R_s)$}\\
       \mathbf{u}_{\textrm{search}} & \textrm{if $k \in \{\mathcal B(\mathbf x_i, R_{\epsilon})\setminus \mathcal B(\mathbf x_i, R_s)\}$}\\
        \mathbf{u}_{\textrm{pursuit}} & \textrm{if $k$ is in pursuit mode}
          \end{array}\right.,\label{eq:sensd}
\end{equation}
in which with $\mathbf u_{\textrm{search}}$ we intend a random motion within the toroid centered in the $i^{\textrm{th}}$ relay controlling the $k^{\textrm{th}}$ mobile sensor. $\setminus$ is the set-minus operator.

$\bullet$ {\em Manipulator agent}: We similarly consider the following logic
\begin{equation}
\mathbf{u}_q=\left\{
    \begin{array}{ll}
       \kappa_{qi}\left( l_{qi}-R_m \right) \hat{\mathbf d}_{qi} -\delta_{i}\dot {\mathbf{x}}_{q} & \textrm{if $q \in \mathcal B (\mathbf x_i, R_m)$}\\
       \mathbf{u}_{\textrm{search}} & \textrm{if $q \in \{\mathcal B(\mathbf x_i, R_{\epsilon})\setminus \mathcal B(\mathbf x_i, R_m)\}$}\\
       \mathbf{u}_{\textrm{manip}} & \textrm{if $q$ is in manipulation mode}
         \end{array}\right.,\label{eq:mand}
\end{equation}

Specifically for the search modes in both \eqref{eq:sensd} and \eqref{eq:mand} we create the following connectivity constraints: at every $\Delta t$ time the communication relay transmits its location to its neighbors. Given $\mathcal V_{\textrm{max}}^c$ the maximum velocity of  each relay, the maximum distance the relay can travel in $\Delta t$ time is $D_{\textrm{max}}^c=\mathcal V_{\textrm{max}}^c \Delta t$. Thus when in search mode, agents are guaranteed to stay in search mode if and only if in an interval of time $\Delta t$ they don't enter inside the region $\{\mathcal B(\mathbf x_i, (R_{\epsilon}-D_{\textrm{max}}^c))\setminus \mathcal B(\mathbf x_i, (R_{s(m)}+D_{\textrm{max}}^c))\}$.
It is important to note that when in pursuit mode, if $k \in \{\mathcal B(\mathbf x_i, R_c)\setminus \mathcal B(\mathbf x_i, R_{\epsilon})\}$ the communication relay $i$ is not attracted anymore toward the target region but switches into following mode with dynamics \eqref{eq:purc} to maintain connectivity to pursuer $k$. 

We can formulate the following theorem that guarantees connectivity among the heterogeneous network at all times:

\begin{theorem}\label{thm:1}
Given an initially connected heterogeneous robotic system made of $\mathcal N_c$ relays, $\mathcal N_s$ mobile sensors, and $\mathcal N_m$ manipulator agents with switching topologies expressed by \eqref{eq:spring}, \eqref{eq:purc}, \eqref{eq:sensd}, \eqref{eq:mand}, the network is guaranteed to maintain connectivity if for an interval of time $\Delta t>0$, each $j\in C_c^i$, $k\in C_s^i$, and $q \in C_m^i$ take goal points $g_{j(k)(q)}^i\in \mathcal B(\mathbf x_i(t),(R_c-\epsilon(\Delta t)))$, with $\epsilon(\Delta t)\geq\mathcal V_{\textrm{max}}^i \Delta t$
\end{theorem}
\begin{proof}
Assuming that the dynamics of each agent are stable or at least stabilizable, if $g_j^i \in\mathcal B(\mathbf x_i(t),(R_c-\epsilon(\Delta t)))$, then $\mathbf x_i(t+\Delta t)\in\mathcal B(\mathbf x_i(t),\epsilon(\Delta t))$ shifting the communication region of $\epsilon(\Delta t)$, thus leaving $\mathbf x_j(t+\Delta t)\in\mathcal B(\mathbf x_j(t),\epsilon(\Delta t))\subset \mathcal B(\mathbf x_j(t),(R_c-\epsilon(\Delta t)))$. Hence, $j$ is always connected to $i$. Note that $\epsilon(\Delta t)$ is an upper (and safe) bound for the sensor $k$ and manipulator $q$ agents since $\mathcal V_{\textrm{max}}^i>\mathcal V_{\textrm{max}}^k=\mathcal V_{\textrm{max}}^q$. 
\end{proof}

\begin{corollary}\label{cor:1}
A sensor agent $k$ connected to the $i^{\textrm{th}}$ relay with dynamical topology \eqref{eq:sensd} is guaranteed to be in {\em search mode} at all times if $|| \mathbf x_k - \mathbf x_i ||>R_s$ and for an interval of time $\Delta t>0$, $g_k^i\in ((\mathcal B(\mathbf x_i(t),(R_c-\epsilon(\Delta t)))\setminus \mathcal B(\mathbf x_i(t),(R_s+\epsilon(\Delta t))))$ with $\epsilon(\Delta t)$ as of Theorem \ref{thm:1}. 
\end{corollary}
\begin{proof}
The proof for this corollary extends from the proof of Theorem \ref{thm:1} and it is based on the geometrical properties of the connectivity constraints imposed in \eqref{eq:sensd}. Corollary \ref{cor:1} applies to the manipulators agents, as well.
\end{proof}
Within Corollary \ref{cor:1} we allow the communication relays to move freely and expand in the environment while the sensor and manipulator agents are in search mode without entering in other regions. Thus $k$ can take a goal $\in (\mathcal B(\mathbf x_i,(R_c))\setminus \mathcal B(\mathbf x_i,(R_c-\epsilon(\Delta t))))$ only if it detects an intruder. In the same way, $q$ can take a goal $\in (\mathcal B(\mathbf x_i,(R_c)))\setminus \mathcal B(\mathbf x_i,(R_c-\epsilon(\Delta t))))$ if and only if it detects the target $\mathcal T$.


%
%
\subsection{Communication Power Control}
In this section we introduce a power control algorithm to maintain a certain SINR level between the agents of the heterogeneous robotic network.

Depending on the distance between the robots as well as path loss, fading, and shadowing, the power received at a certain mobile sensor (or manipulator) $k$ that is transmitted by a relay $i$ is attenuated by a gain $g_{ik}=K\left(\frac{l_0}{l_{ik}}\right)^{\beta}+\frac{\psi_{ik}}{P_i}$ where $K$ is a constant that depends on antenna properties and channel attenuation, $l_0$ is a reference distance, and $l_{ik}$ is the separation distance between robots $i$ and $k$. $\beta$ is the path-loss exponent and finally $\psi_{ik}$ is an additional attenuation due to shadowing and that is usually a log-normal distributed random variable \cite{goldsmith2005wireless}. $i$ can communicate with $k$ provided its (SINR) is above a certain threshold $T$. Thus the goal is to find whether there exists an assignment of power levels and distance between robots so that each robot's SINR is acceptable. Since $g_{ik}$ depends on the distance separation $l_{ik}$ between node $i$ and $j$, instead of regulating the power, we can think of changing $l_{ik}$ to maintain the SINR above a desired value. Therefore, we can implement the following algorithm
\begin{equation}\label{eq:pc}
{\textrm {find}} \sum_{k\in \mathcal C_c^i} g_{ik} \;\; \forall i=1,\ldots, \mathcal N_c, 
\end{equation}

subject to:
\begin{equation}\label{eq:conpc}
\begin{array}{ll}
\left\{ \frac{P_i g_{ik}}{\sum_{j\neq i, j \in \mathcal B(\mathbf x_k, R_s)} P_j g_{jk}+\nu_k} \right\}\geq T,\\
0<P_i \leq P_{\textrm{max}},  \\
l_{ik}\geq l_{\textrm{min}} & \forall i=1,\ldots, \mathcal N_c, \forall k\in \mathcal C_c^i.
\end{array}
\end{equation}
where $l_{\textrm{min}}$ is the minimum distance separation between any agent in order to avoid collision. $P_{\textrm{max}}$ is an upper bound for the maximum power each agent can have, $P_i$ is the power level of agent $i$, and $\nu_k$ the receiver noise at the $k^{\text{th}}$ sensor. For simplicity's sake we assume that the $\nu_k$ can be neglected.
In other words, by implementing this algorithm we guarantee a certain quality of service (QoS) among the robotic team and adjust the received power at $k$ through mobility.

\subsection{Sensing \& Manipulation}\label{sec:6}
In this work we consider two distinct sensing behaviors: a {\em vision capability} and a {\em obstacle avoidance characteristic}, as depicted in Fig.\ref{fig:sensor}(a-b). The former implies the use of camera like systems in which we are able to perceive different features in the environment and for instance recognize friends and foes, targets and other characteristics of the environment. Within the obstacle avoidance, we consider laser range finder type sensors that are capable of measuring distances with high precision and thus can be employed for navigation. Finally the manipulator agents are equipped with a planar arm to lift or move objects (\ref{fig:sensor}(c)) .
\begin{figure}
\begin{center}
\subfigure[] {\includegraphics[width=0.350\textwidth]{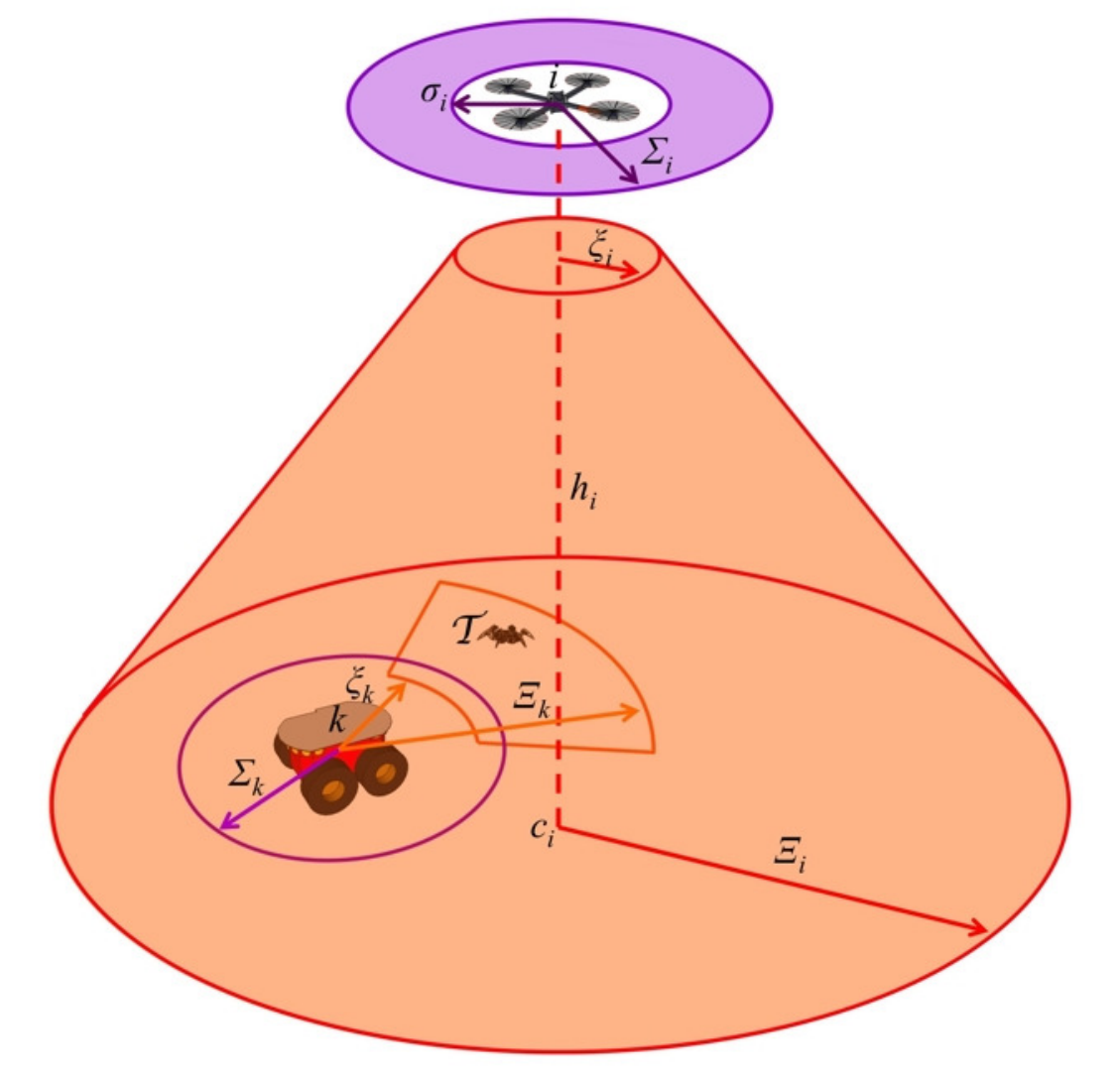}}
\subfigure[] {\includegraphics[width=0.350\textwidth]{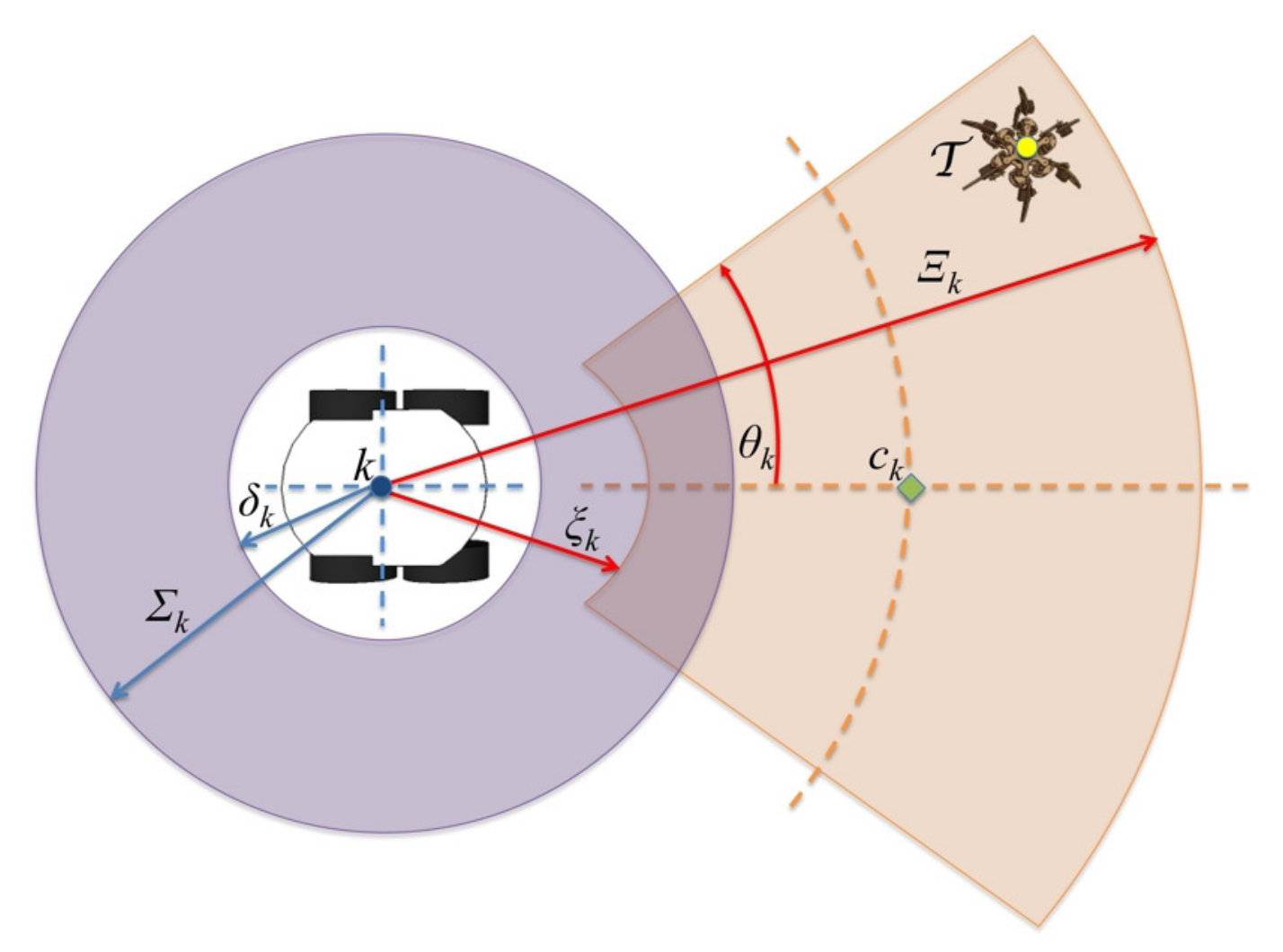}}
\subfigure[] {\includegraphics[width=0.280\textwidth]{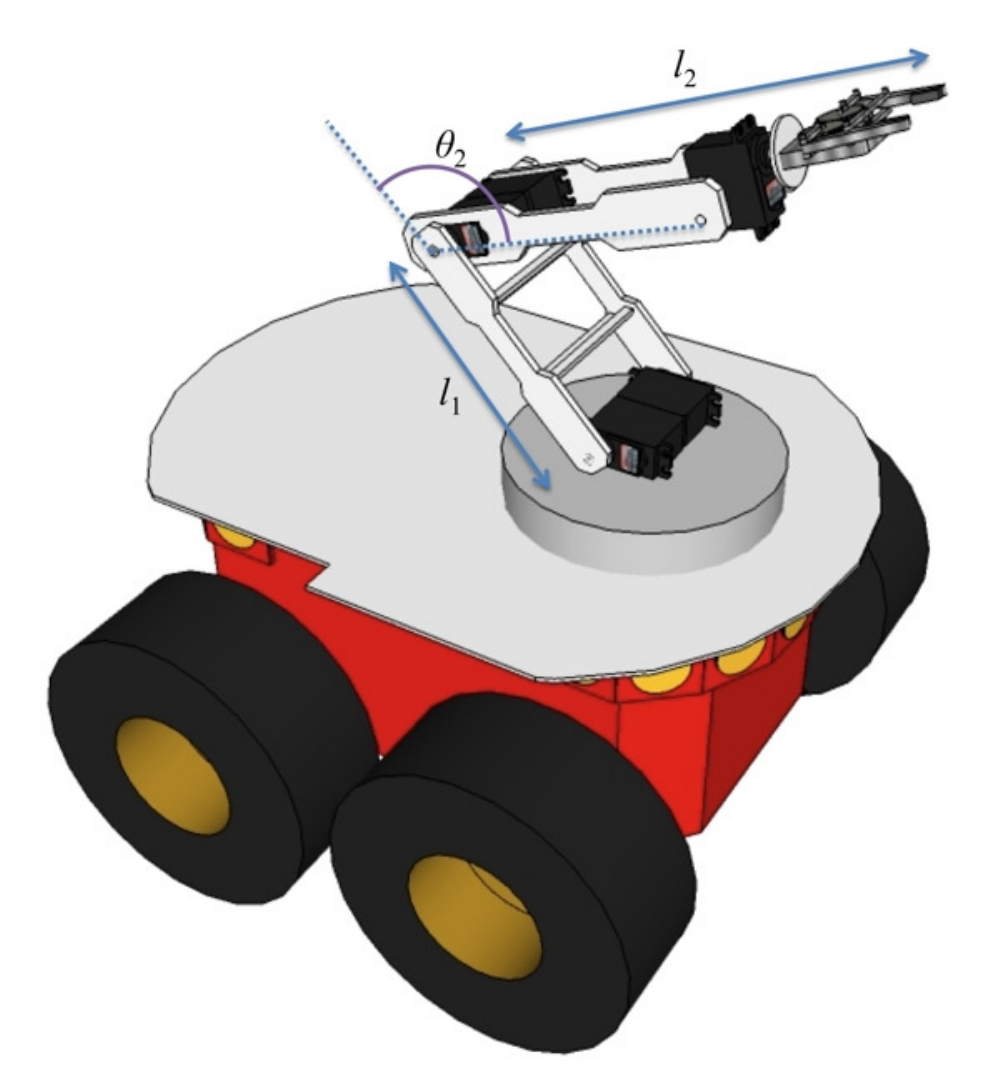}}

\end{center}
\caption{ Representation of the sensing capabilities for the aerial relay and mobile sensor. (a) The aerial relay $i$ hovers at a certain height $h_i$ and has a toroidal sensing for obstacle avoidance and a conical field of view over the ground. (b) The sensor agent $k$ (and also the manipulator $q$) has a toroidal sensing for obstacle avoidance and a limited field of view in front of it. $\mathcal T$ is a target of interest. (c) Representation of the manipulator configuration used in this work. For ease, here we consider a two link planar arm with end effector, installed on the top, frontal position of the robot.}
\label{fig:sensor}
\end{figure}

\subsubsection{Vision Detection}
In the scenario envisioned in this paper, each robot has some degree of vision capability. The aerial relay can see a large area but with low resolution, while the sensor and the manipulator agents on the contrary can perceive a smaller area but with higher resolution.
Following Fig.\ref{fig:sensor}, we use the following probability of detection field for the aerial relays
\begin{equation}\label{eq:vision_r}
\mathcal{S}_i(\mathbf x_{\mathcal T})=\left\{
    \begin{array}{ll}
    \textrm{N}\left( c_i, \varphi_i^2\right)& \textrm{if \,$||\mathbf x_{\mathcal T}-c_i||\leq \Xi_i$ and $h_i=const.$} \\
    0 & \textrm{otherwise}
     \end{array}\right.
\end{equation}
where $\mathbf x_{\mathcal T}$ is the state (\ie the position) of a target ${\mathcal T}$ located in a 3D workspace. $\textrm{N}(c_i, \varphi_i^2)$ is the field of view normal distribution centered in $c_i$ with variance $\varphi_i^2$.   $||\cdot||$~is the euclidean distance norm and $\Xi_i$ is the maximum vision range for the $i^{\textrm{th}}$ aerial relay. Model \eqref{eq:vision_r} holds if the quadrotor is hovering at a constant altitude (hence $h_i=const.$). By using this model, the probability of detection is higher moving toward the centroid of the field of view of $i$. Thus if a target ${\mathcal T}$ in position $\mathbf x_{\mathcal T}$ is such that $||\mathbf x_{\mathcal T}-c_i||\leq \Xi_i$, the probability of detection (pr) is given by
\begin{equation}\label{eq:pr_r}
\textrm{pr}(\mathbf x_{\mathcal T})=\textrm{N}(c_i, \varphi_i^2)=\frac{1}{\varphi_i\sqrt{2\pi}}e^{-\frac{(\mathbf x_{\mathcal T}-c_i)^2}{2\varphi_i^2}}
\end{equation}

A centroid motion scheme, centering the target in this field of view is discussed in the following section
 
For the sensor and manipulator agents, similarly we use the following constraint
\begin{equation}\label{eq:vision_s}
\mathcal{S}_k(\mathbf x_{\mathcal T})=\left\{
    \begin{array}{ll}
    \textrm{N}\left( \left(\frac{ (\Xi_k-\mathbf x_k)-(\xi_k-\mathbf x_k)}{2}\right), \varphi_k^2\right)& \textrm{if \,$\mathbf x_{\mathcal T}\in \theta_k(||\Xi_k-\mathbf x_k||^2-||\frac{(\Xi_k+\xi_k)}{2}-\mathbf x_k||^2)$.} \\
    \textrm{N}\left( \left(\frac{ (\Xi_k-\mathbf c_k)-(\xi_k-\mathbf c_k)}{2}\right), \varphi_k^2\right)& \textrm{if \,$\mathbf x_{\mathcal T}\in \theta_k(||\frac{(\Xi_k+\xi_k)}{2}-\mathbf x_k||^2-||\xi_k-\mathbf x_k||^2)$.} \\
    0 & \textrm{otherwise}
     \end{array}\right.,
\end{equation}
where $\textrm{N}(\cdot,\cdot)$ has the same form of \eqref{eq:pr_r}. $\Xi_k$ and $\xi_k$ are the maximum and minimum distances perceptible by agent $k$, respectively. Finally $2\theta_k$ is the viewing angle of $k$, as represented in Fig.\ref{fig:sensor}(b). 
\subsubsection{Pursuit \& Evasion}
If a mobile sensor $k$ detects an adversarial opponent $\mathcal T$ inside its field of view, it switches from {\em search mode} into {\em pursuit mode} \eqref{eq:sensd} with the following dynamics
\begin{equation}
\mathbf u_{\textrm{pursuit}}=\left\{
\begin{array}{ll}
\alpha (\mathbf{T}_k^{\mathcal W}(\mathbf x_{\mathcal T}^k-c_k^k)) & \textrm{if \,$\mathbf x_{\mathcal T} \in \theta_k(||\Xi_k-\mathbf x_k||^2-||\frac{(\Xi_k+\xi_k)}{2}-\mathbf x_k||^2)$}\\
\alpha(\mathbf x_{\mathcal T}-\mathbf x_k) & \textrm{if \,$\mathbf x_{\mathcal T}\in \theta_k(||\frac{(\Xi_k+\xi_k)}{2}-\mathbf x_k||^2-||\xi_k-\mathbf x_k||^2)$}       
\end{array}\right.,\label{eq:purs}
\end{equation}
in which $\mathbf{T}_k^{\mathcal W}$is the transformation matrix that converts the $k^{\textrm{th}}$ robot frame into the world $\mathcal W$ frame. $\mathbf{T}_k^{\mathcal W}=\mathbf R_k^{\mathcal W} \mathbf D_k^{\mathcal W}$ where $\mathbf R_k^{\mathcal W}$ is the rotation matrix and $\mathbf D_k^{\mathcal W}$ is the translation matrix, \cite{spong2006robot}. $\mathbf x_{\mathcal T}^k$ and $c_k^k$ are the position of the target and field of view centroid in the $k^{\textrm{th}}$ robot frame, respectively.

Thus, within the first equation in \eqref{eq:purs} we navigate the centroid $c_k$ of the field of view of $k$ toward $\mathbf x_{\mathcal T}$. Once $\mathcal T$ is within the region of the constraint in the second equation of \eqref{eq:purs}, $k$ is guided toward the evader through an attractive potential force. Here we assume that based on the velocities of both $k$ and $\mathcal T$, $\mathcal T$ is capturable if \,$\mathbf x_{\mathcal T}\in \theta_k(||\frac{(\Xi_k+\xi_k)}{2}-\mathbf x_k||^2-||\xi_k-\mathbf x_k||^2)$ (that is the constraint in the second equation of \eqref{eq:purs}).
\subsubsection{Obstacle Avoidance}
For the obstacle avoidance effect the reader is referred to the toroidal shapes in Fig.\ref{fig:sensor}.
The workspace, $\mathcal{W}$, is populated with $\mathcal N_o$ fixed
polygonal obstacles $\{{O}_1,\ldots,{O}_{\mathcal N_o}\}$, whose geometries and positions are assumed unknown.
In order to avoid obstacles we model a ray field of view around the agents, similarly to a laser range finder footprint, and we create a repulsive potential whose value approaches infinity as the robot approaches the obstacle, and goes to zero if the robot is at a distance greater than $\Phi_i$ or smaller than $\phi_i$ from the obstacle. 
\begin{equation}
W_{O,i}=\left\{
    \begin{array}{ll}
        \frac{1}{2} \eta_{i} \left(\frac{1}{\rho({\mathbf x}_{i})}-\frac{1}{\rho_{0}}\right)^2& \textrm{if $\phi_i\leq \rho({\mathbf x}_{i}) \leq \Phi_i$} \\
        0 & \textrm{if $\rho({\mathbf x}_{i}) > \Phi_i$ or $\rho({\mathbf x}_{i}) <\phi_i$}
  \end{array}\right.,\label{eq:conob}
\end{equation}
where $\rho({\mathbf x}_{i})$ is the shortest distance between the agent and any detected obstacle in the workspace and $\eta_{i}$ is a constant.

The repulsive force is then equal to the negative gradient of $W_{O,i}$
For simplicity sake, here we assume that the aerial relay, sensor and manipulator agents have all the same obstacle avoidance constraint, as depicted in Fig.\ref{fig:sensor}(a-b).
\subsubsection{Manipulability}
For the manipulation behavior we assume that $\mathcal N_m$ agents are equipped with an articulated arm having $\mathcal N_l$ links. From a classical robotics book, \cite{spong2006robot}, it is well known that the {\em manipulability metric} offers a quantitative measure of the relationship between differential change in the end-effector pose relative to differential change in the joint configuration. In this work, we use this concept to define the best configuration of the manipulator agent such that when a certain target object needs to be handled, the manipulability measure is maximized.

Let us define the Jacobian relationship
\begin{equation}
\zeta=J\dot{a}
\end{equation}
that specifies the end-effector velocity that will result when the joint move with velocity $\dot{a}$.

If we consider the set of joint velocities $\dot{a}$ such that $||\dot{a}||^2=\dot{a_1}^2+\dot{a_2}^2+\ldots \dot{a_{\mathcal N_l}}^2\leq1$, then we obtain
\begin{equation}\label{eq:man1}
||\dot{a}||^2=\zeta^T(JJ^T)^{-1}\zeta=(U^T\zeta)^T\Sigma_m^{-2}(U^T\zeta)
\end{equation}
in which we have used the singular value decomposition SVD $J=U \Sigma V^T$ \cite{spong2006robot}.

If the Jacobian is full rank (rank $J=m$), \eqref{eq:man1} defines the manipulability ellipsoid. It is easy to show that the manipulability measure is given by
\begin{equation}\label{eq:man2}
\mu=\sigma_1\sigma_2\ldots \sigma_m
\end{equation}
with $\sigma_i$ the diagonal elements of $\Sigma$.

For convenience in this work we consider that each manipulator robot is equipped with a two-link planar arm (Fig.\ref{fig:sensor}(c)) in which the manipulability measure is given by
\begin{equation}\label{man3}
\mu=l_1 l_2 |\sin \theta_2|
\end{equation}
where $l_1$ and $l_2$ are the length of the two links of the manipulator and $\theta_2$ is the angle between the two links. Therefore, the highest manipulability measure is obtained when $\theta_2=\pi/2$. 

\subsubsection{Target Manipulation}
If a manipulator agent $q$ detects a fixed target $\mathcal D$ that needs manipulation (\eg lifting, moving, grabbing, etc.) then we apply a similar law as in \eqref{eq:purs} with the only difference that once the target is detected, we compute the configuration of the manipulator arm such that we obtain the highest manipulability measure $\mu$. $\mu$ will translate into a certain spatial configuration and position of the end effector $\mathbf x_{\zeta}$. Finally, $\mathbf x_{\zeta}$ becomes the input for the controller. Thus the control law for $q$ in {\em manipulation mode} becomes
\begin{equation}\label{eq:manipulate}
\mathbf u_{\textrm{manip}}=\alpha (\mathbf{T}_q^{\mathcal W}(\mathbf x_{\mathcal D}^q-\mathbf x_{\zeta}(\mu))) \,\,\, \textrm{if \,$\mathbf x_{\mathcal D} \in \theta_q(||\Xi_q-\mathbf x_q||^2-||\mathbf x_{\zeta}(\mu)-\mathbf x_q||^2)$}
\end{equation}
%
%
%
%
%
%
%
\subsection{Simulation Results}\label{sec:7}

In the simulation of Fig.~\ref{fig:sim2} we assemble together all the pieces descried in the previous sections and consider a search and rescue/pursuit-evasion scenario. A heterogeneous system made of the same number and type of agents as in the first simulation, explores an environment in search of a target $\mathcal D$ that needs to be manipulated, while maintaining connectivity with a fixed base station. The target is protected by an opposing player $\mathcal T$ that circles around its perimeter. $\mathcal T$ tries to capture the manipulator agent $q$, while avoiding any mobile sensor $k$. If a mobile sensor $k$ detects $\mathcal T$, it switches into pursuit mode to capture the opponent. Here we assume that $\mathcal V_{\textrm{max}}^{k}>\mathcal V_{\textrm{max}}^{\mathcal T}>\mathcal V_{\textrm{max}}^{q}$.
 \begin{figure}[ht!]
\begin{center}
\subfigure[] {\includegraphics[width=0.325\textwidth]{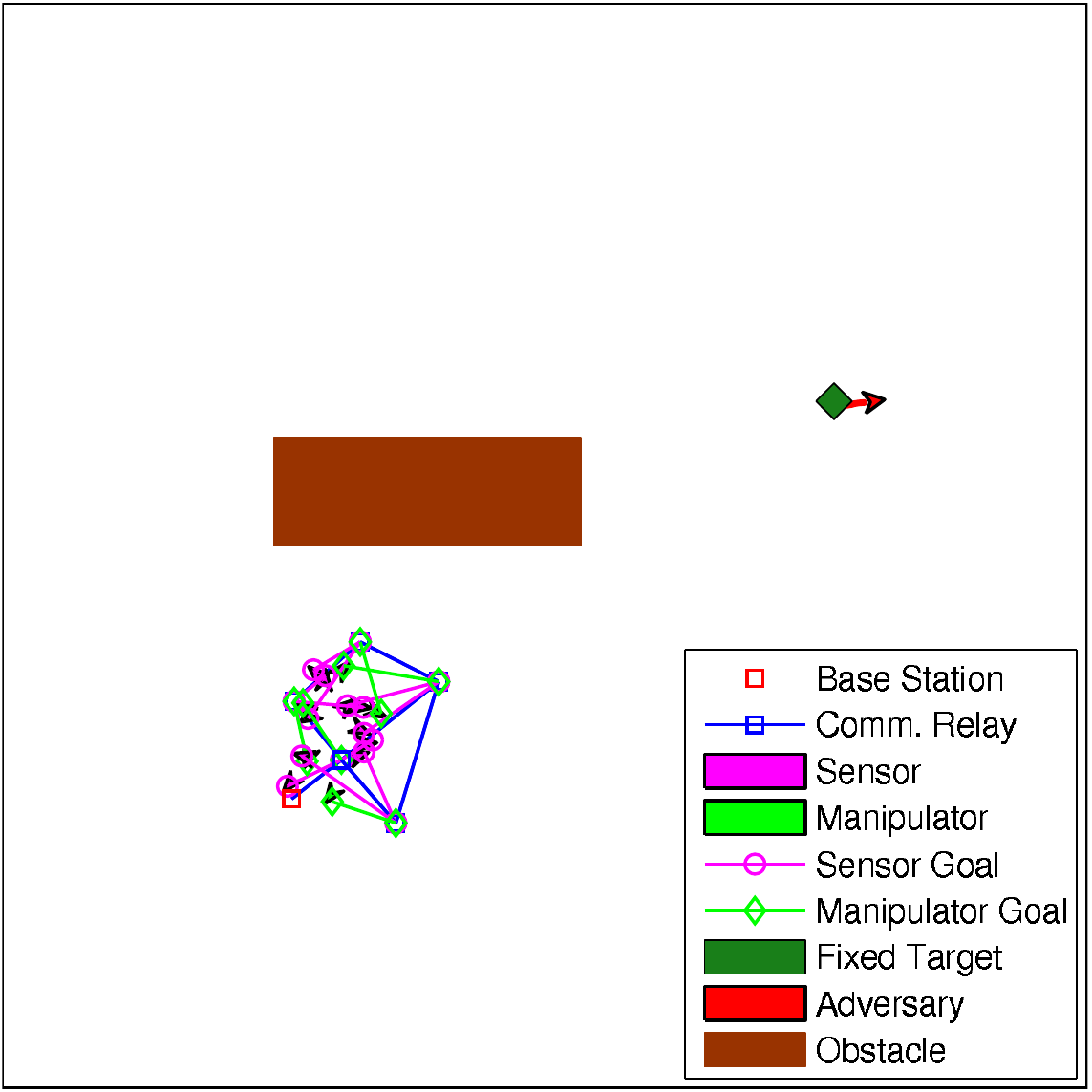}}
\subfigure[] {\includegraphics[width=0.325\textwidth]{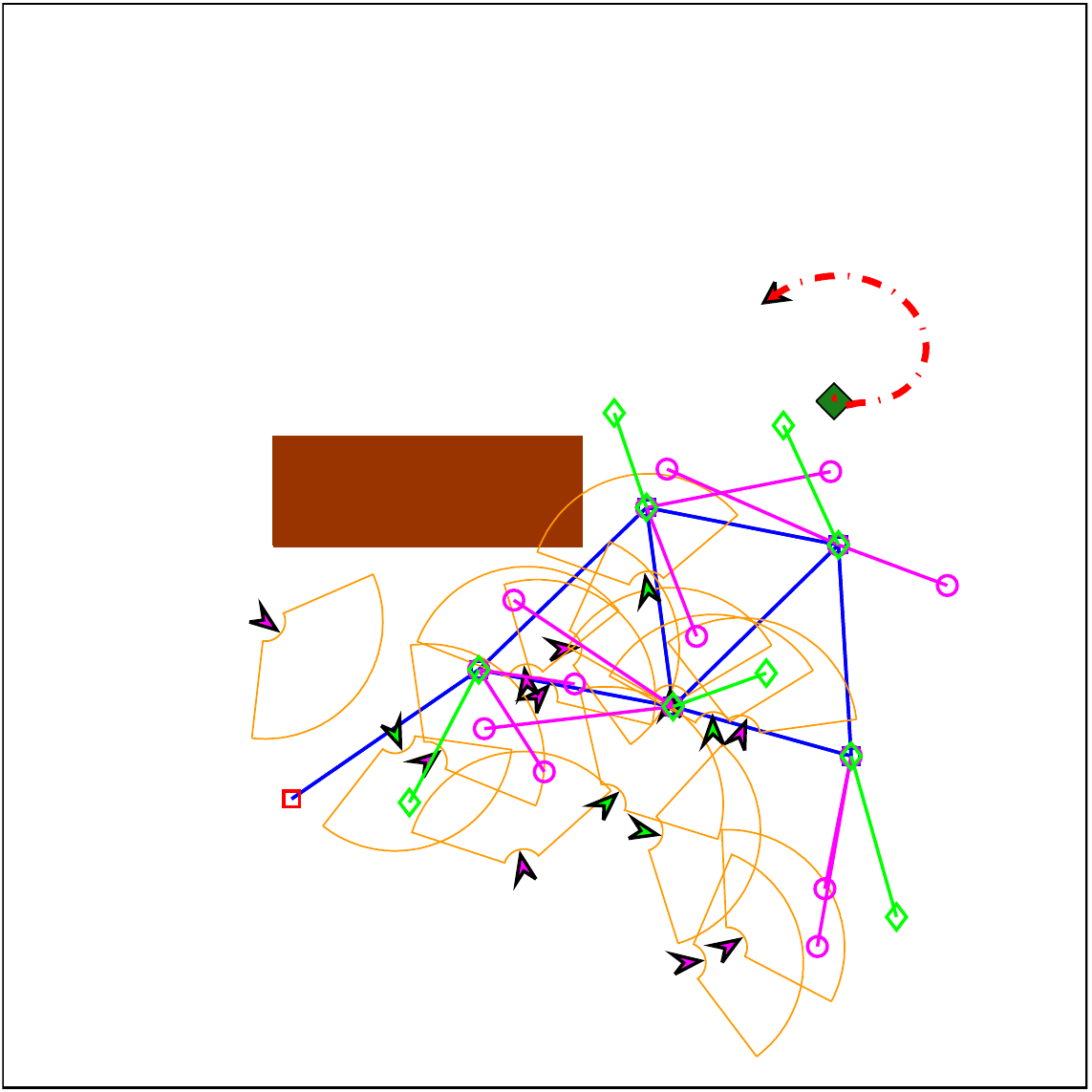}}
\subfigure[] {\includegraphics[width=0.325\textwidth]{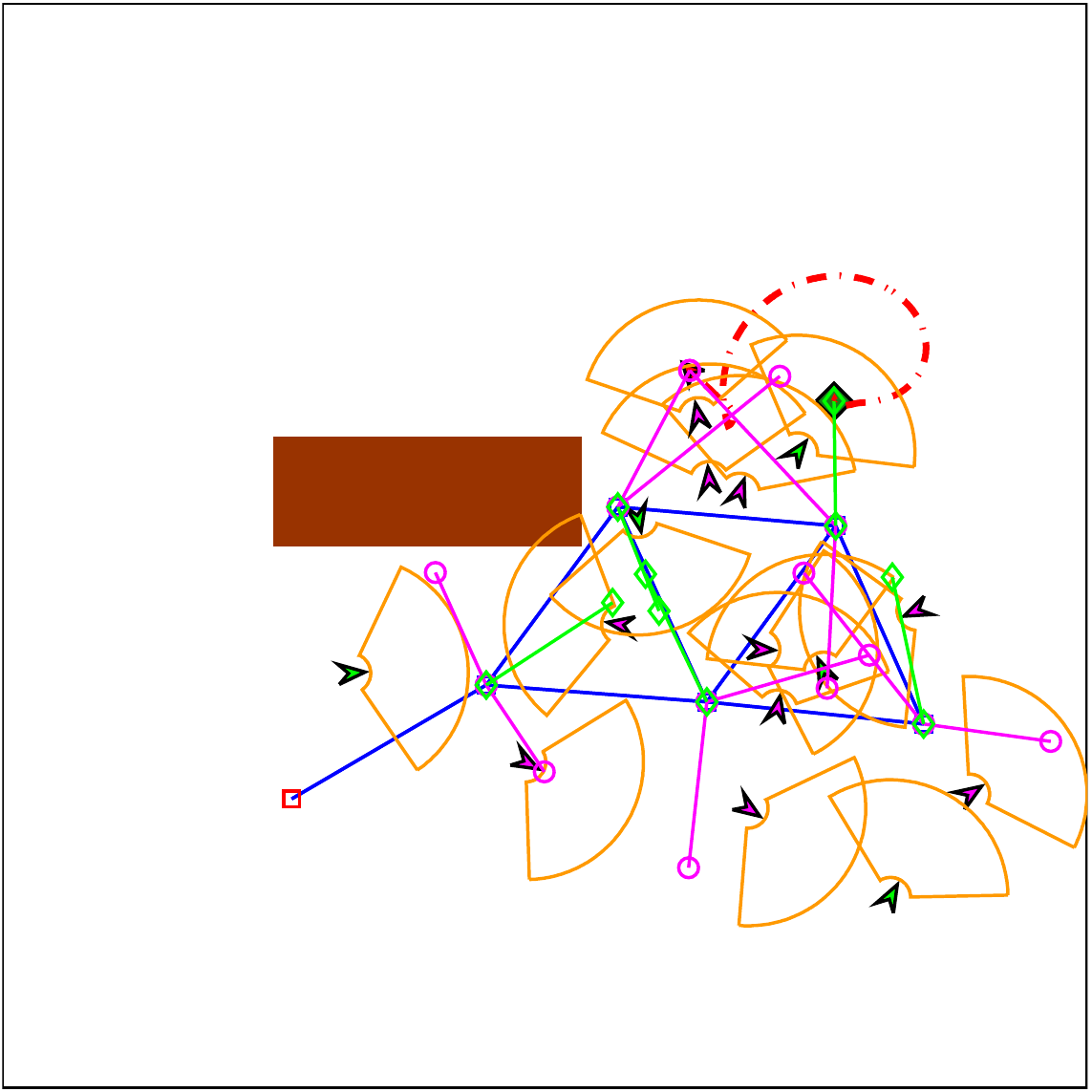}}

\end{center}
\caption{\label{fig:sim2} Pursuit-evasion simulation. (a) The heterogeneous system is in spring-mass mode. (b) The sensor and the manipulator agents are in search mode while the communication relays expand the network in the environment. (c) A sensor detects and pursue an adversarial player (top right of the figure making a circular trajectory), while a manipulator moves toward the fixed target. The video of the simulation is available at {\em http://marhes.ece.unm.edu/index.php/ROBOTICA2013}.} 
\end{figure}

Specifically, during the simulation, while the consensus algorithm \ref{alg:Alg} is run to equilibrate the network, the agents are attracted to a region where the target is located. The relays maintain network connectivity and enforce the power control algorithm \eqref{eq:pc} keeping the SINR above a certain threshold (Fig.~\ref{fig:sinrf}(a)). 
In Fig.~\ref{fig:sinrf}(b) it is plotted the case in which the agents don't follow the PC algorithm. Specifically in this case the SINR can take values below the threshold, obtaining a lower quality of communication. In Fig.~\ref{fig:sim2}(b) the sensor and manipulator robots switch into search mode and explore the area surrounding their assigned communication relay. Finally, in Fig.~\ref{fig:sim2}(c) a mobile sensor detects the opposing player and switches into pursuit mode while a manipulator moves toward the fixed target.

Additionally, through equivalence comparisons, we found that the average coverage for a heterogeneous system is 63$\%$ while for the homogeneous scenario is reduced to 22$\%$ of the workspace. Therefore, by decoupling the tasks of sensing and relaying communication, and by imposing the connectivity constraints for a heterogeneous system we can cover a larger area with a limited number of mobile agents.
\begin{figure}[ht!]
\centering
{\includegraphics[width=0.825\textwidth]{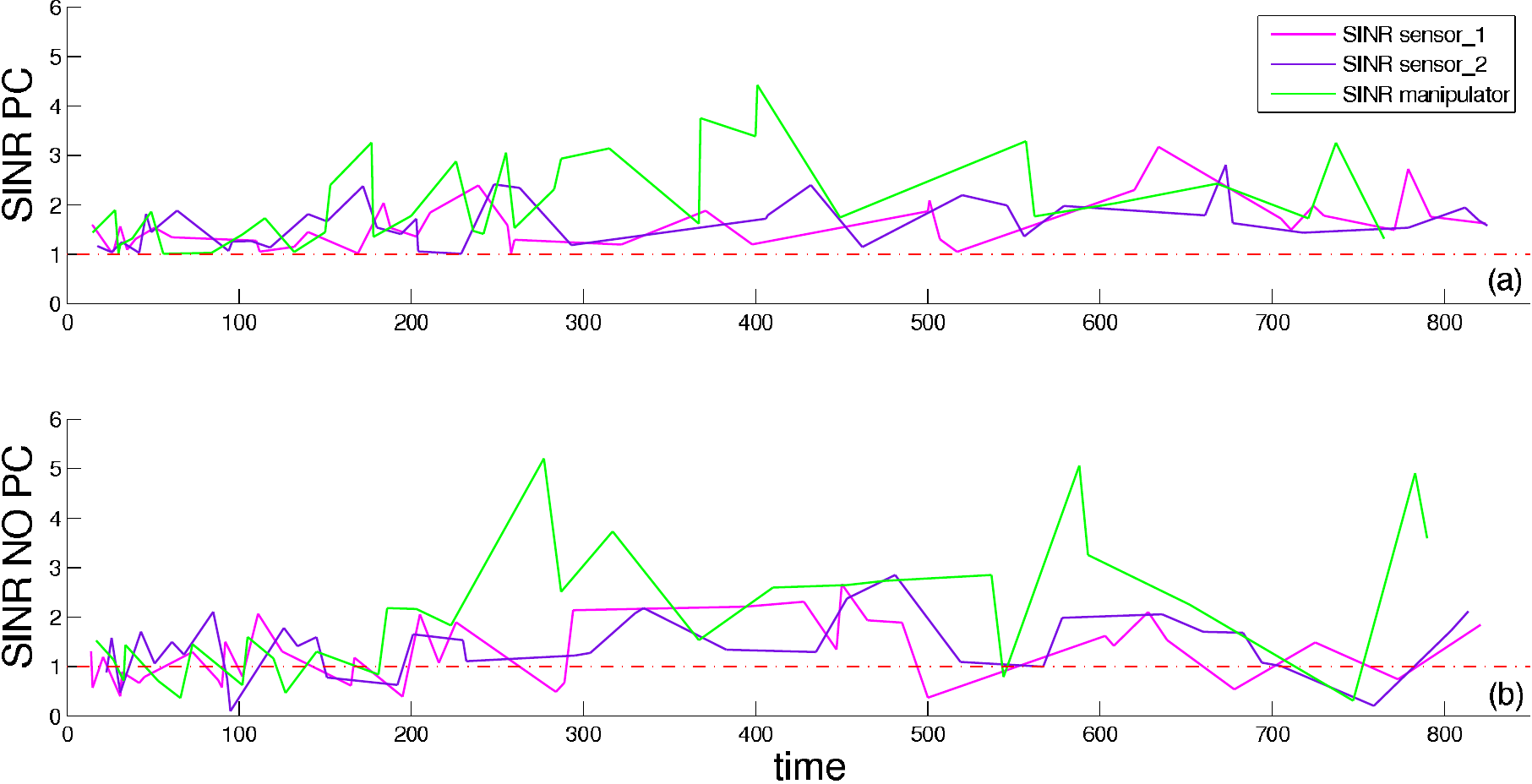}}
\caption{\label{fig:sinrf} Comparison between SINR (dB) with Power Control (top) and without Power Control (bottom) for two mobile sensors and one robotic manipulator assigned to one of the five relays in Fig.~\ref{fig:sim2}.}
\end{figure}

\section{Biologically--Inspired Coordination of Heterogeneous Robotic Systems}\label{sec:4}

In the previous section we demonstrated our heterogeneous coordination framework applied to a simulated pursuit--evasion scenario with a single mobile target. We now apply our framework to a test bed with physical ground and aerial robots searching for multiple stationary resources. In this scenario, robots are programmed to locate and collect as many resources as possible within a fixed amount of time. Our biologically-inspired ground robots (depicted in Fig. \ref{fig:img1}(a)) explore the experimental area using biased random walks. The bias of a robot's walk varies when the robot is informed about resource locations. This precludes an analytical prediction for how long it will take to retrieve resources with or without the aerial drone. Instead we conduct a series of experiments with and without coordination between the ground robots and an aerial drone, demonstrating the value of the heterogeneous coordination framework on a resource collection problem.

Each experimental trial on a 2.5 m square indoor tile surface runs for approximately 30 minutes using two ground robots and a single aerial drone. We begin each trial by randomly placing 32 barcode--style QR tags in clusters throughout the search area, either in one large cluster of 32, or in four smaller clusters of 8. The ground robots are capable of detecting individual tags as they explore the search area over the length of the experiment. Each tag cluster is marked with a corresponding roundel--style pattern which can be recognized by the drone's built--in vision tracking system. The aerial drone assists the ground robots by guiding them toward the tag clusters, thereby increasing the probability of tag detection.

In accordance with our definition of a heterogeneous coordination framework, the ground robots and aerial drone have different motion dynamics and sensing constraints. The aerial drone is fast-moving and holonomic with coarse-grained, wide-ranging vision sensors, whereas ground robots are slow--moving non--holonomic entities with fine--grained, short--range sensors. Equation \eqref{eq:vision_r} constrains the vision sensing capabilities of both types of robots, although the maximum vision range $\Xi$ of the aerial drone is much larger than that of the ground robots. As a result, the drone has a high probability of detecting a region of QR tags by recognizing the associated roundel pattern, but has zero probability of detecting an individual tag $\mathcal T$ at position $\mathbf x_{\mathcal T}$. Conversely, the ground robots have a high probability of detecting tag $\mathcal T$ but a very low chance of finding it on their own.

In this way, each type of robot specializes in recognizing one particular type of symbol, and the spatial association of QR tag clusters together with roundel patterns bridges this gap to facilitate cooperation. We exploit this cooperation to optimize the probability of detecting resources, as in Equation \eqref{eq:pr_r}, and therefore we expect the rate of resource collection to increase relative to non-cooperative or homogeneous robot teams.

\subsection{Cooperative Task and Search Algorithm}

The aerial vehicle executes a deterministic grid search over the experimental area, described in Algorithm \ref{alg:search_algorithm_aerial}. Resource locations identified by the aerial robot are transmitted to the ground robots as pheromone-like waypoints. Communication connectivity is constrained as in Section 2.2, such that the aerial vehicle is assumed to have a much wider communication range than the ground robots. The entire experimental area is considered to be a valid range for two--way communication with the drone, whereas transmission of resource locations to ground robots is permitted only when returning to the central base station `nest'.

Our ground robots execute an ant--inspired central--place foraging task to search for resources simultaneously with the aerial vehicle, described in Algorithm \ref{alg:search_algorithm_ground}. They move in a correlated random walk with direction $\theta$ at time \emph{t} drawn from a normal distribution centered around direction $\theta_{t-1}$ and standard deviation ${SD = \omega+ \gamma/t_s^\delta}$. $\omega$ determines the degree of turning during an \textbf{uninformed} search. In a search \textbf{informed} by memory or communication,  $\gamma/t_s^\delta$ determines an initial additional degree of turning which decreases over time spent searching. This mimics biological ants' tight turns in an initially small area that expand to explore a larger area over time \cite{moses2011how}.

A total of 13 parameters controlling the exploratory process of the ground robots are evolved in an agent--based model (ABM) guided by genetic algorithms (GA). Three of these parameters govern traveling behavior to and from the nest, three control turning during random walks, and seven affect the probability of returning to a found resource location using memory or communication. We designed the ABM to replicate the constraints of the robot hardware, and to model the physical environment in which the robots search. We briefly address a subset of the parameters here (\eg $\omega$, $\gamma$, and $\delta$; also see Algorithm \ref{alg:search_algorithm_ground}); specifications of the full parameter set, the ABM, and it's relationship with the ground robots are detailed in previous work \cite{hecker2012formica}.

\begin{algorithm}[t]
\caption{Aerial Robot Search Algorithm}
\label{alg:search_algorithm_aerial}
\begin{algorithmic}
\ENSURE {Drone placed at starting location}
\STATE Take off and hover \\
Follow search pattern in Fig. \ref{fig:search_algorithm_aerial}\subref{fig:complete_search} by adjusting pitch and roll\\
	\IF{roundel found} \STATE{
		Transmit message to central server\\
		Server records drone's location as a non-decaying pheromone waypoint
	} \ENDIF
\end{algorithmic}
\end{algorithm}

\begin{algorithm}[!ht]
\caption{Ground Robot Search Algorithm}
\label{alg:search_algorithm_ground}
\begin{algorithmic}
\STATE Disperse from nest to random location
\WHILE{running experiment} \STATE{
	Conduct uninformed correlated random walk\\
	\IF{tag found} \STATE{
		Count number of nearby tags $d$ at current location $l_f$\\
		Return to nest with tag
		\IF{d $> t_p$\footnoteremember{myfootnote}} \STATE{
			Communicate $l_f$ as pheromone-like waypoint to central server\\
			Server transmits pheromones to robots at nest\\
			Server decays pheromone over time
		}
		\ELSE
			\IF{d $> t_f$\footnoterecall{myfootnote}} \STATE{
				Return to $l_f$\\
				Conduct informed correlated random walk
			}
			\ELSE \STATE{
				Request pheromone  from server
				\IF{pheromone available \AND d $<t_h$\footnoterecall{myfootnote}} \STATE{
					Drive to received waypoint $l_p$\\
					Conduct informed correlated random walk
				}
				\ELSE \STATE{
					Choose new random location
				} \ENDIF
			} \ENDIF
		\ENDIF
	} \ENDIF
} \ENDWHILE
\end{algorithmic}
\footnoterecall{myfootnote}\footnotesize{Parameter evolved in simulation using GA, identical across all robots}
\end{algorithm}

\subsection{Test bed}

The ground robots are Arduino--based iAnts \cite{hecker2012formica}, autonomous differential drive rovers with dynamics expressed in \eqref{eq:nh}. They are guided only by a heuristic search algorithm with occasional communication via pheromone--like waypoints transmitted between robots through a central server (Fig. \ref{fig:search_algorithm_ground}). The iAnts use onboard iPods that provide wireless communication and two video cameras, in addition to a compass and ultrasonic rangefinder.
\begin{figure}[ht!]
	\centering
	\subfigure[]{
	\includegraphics[scale=1]{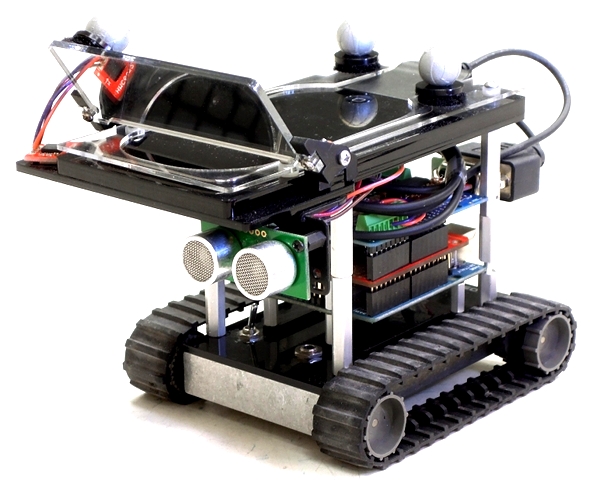}
	\label{fig:robot}
	}
	\subfigure[]{
	\includegraphics[scale=.37]{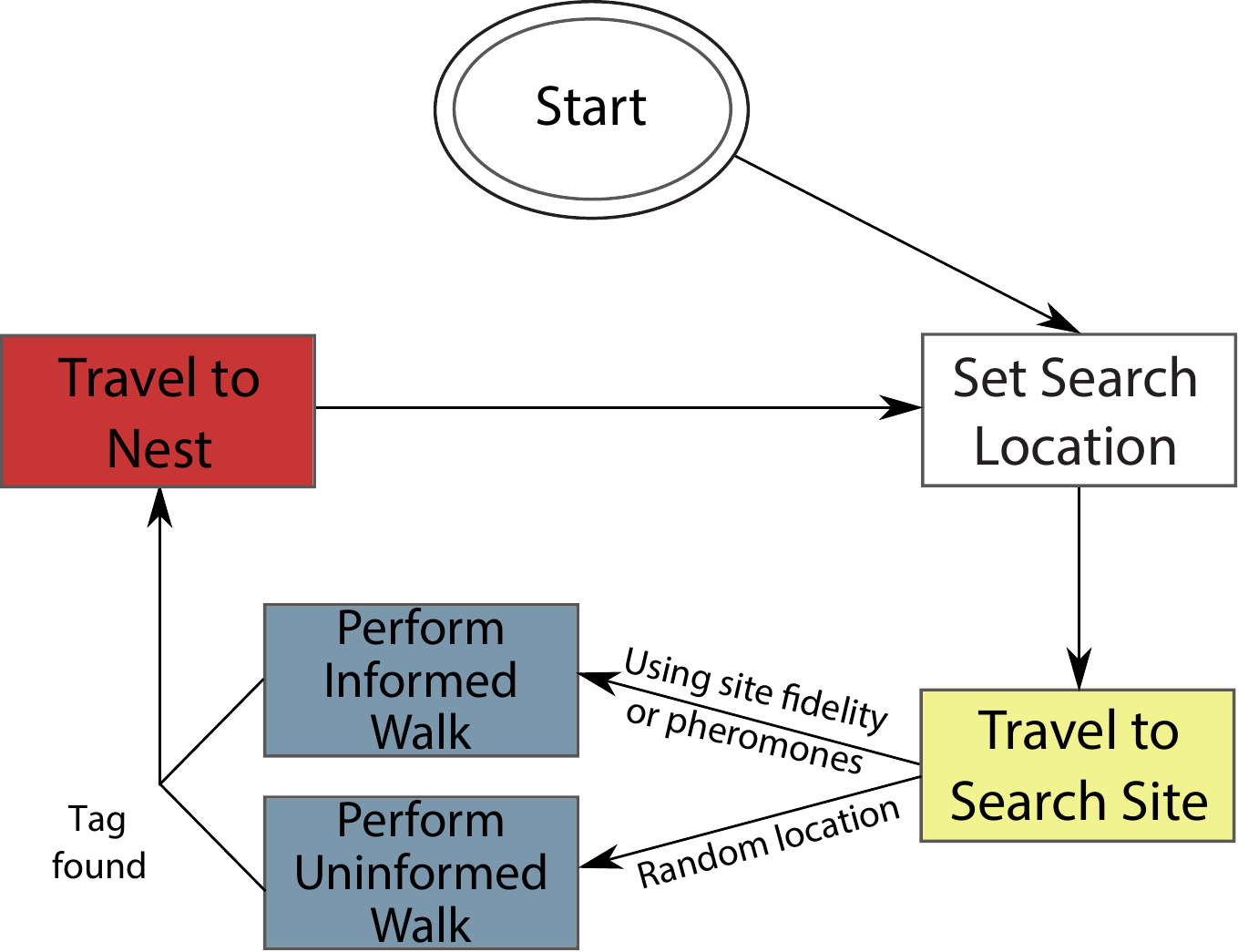}
	\label{fig:state_diagram}
	}
	\subfigure[]{
	\includegraphics[scale=.53]{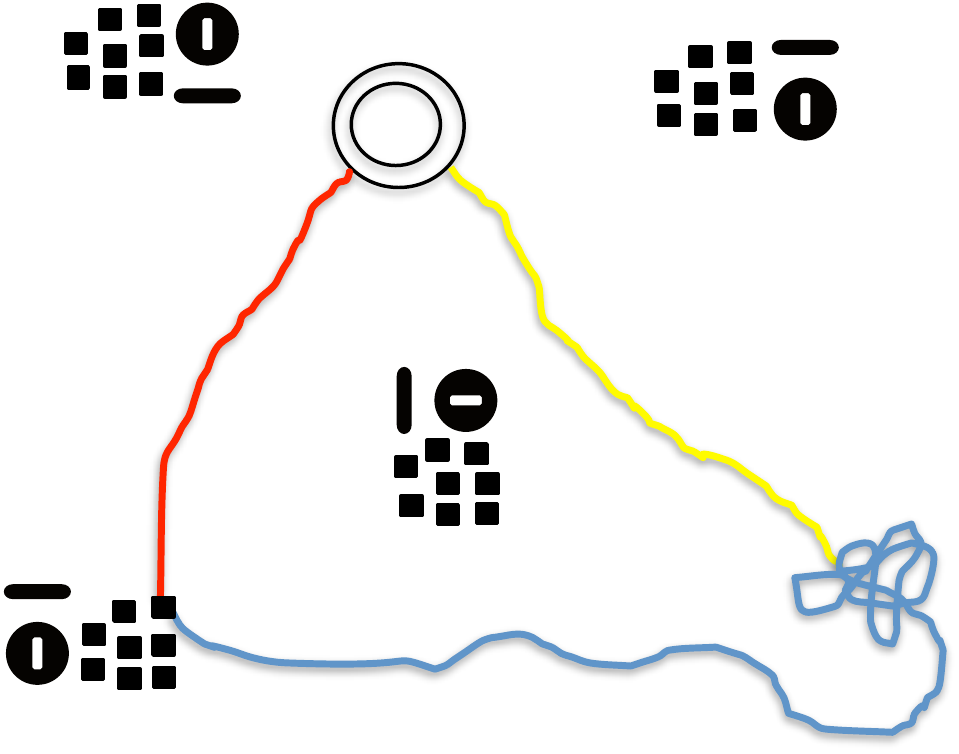}
	\label{fig:tag_search}
	}
	\caption{\protect\subref{fig:robot} One of the iAnt robots used in our heterogeneous experiments; \protect\subref{fig:state_diagram} block diagram with the sequence of operations performed during the search; and \protect\subref{fig:tag_search} a pictorial representation of the search algorithm. The iAnt begins its search at a globally shared central nest site (double circle) and \textbf{sets a search location}. The robot then \textbf{travels to the search site} (yellow line). Upon reaching the search location, the robot \textbf{searches for tags} (blue line) until tags (black squares) are found. After searching, the robot \textbf{travels  to the nest} (red line).}
	\label{fig:search_algorithm_ground}
\end{figure}

Our ground robots coordinate with an aerial vehicle, the Parrot AR.Drone radio--controlled quadrotor (Fig. \ref{fig:search_algorithm_aerial}). The AR.Drone contains an onboard inertial measurement unit (IMU) to control holonomic flight through 6 degrees of freedom (DOF), as well as two video cameras with built--in object recognition, and uses an ultrasound telemeter to maintain a consistent hover state.

A central server facilitates all network traffic between the robots, both ground and aerial. The server also tracks robot locations throughout the experiment for data logging; occasional two--way communication allows virtual pheromones to direct the iAnts to previously found tag locations. 

We additionally use a Vicon motion capture system to track the iAnts' and AR.Drone's ground truth position and orientation. Lateral and longitudinal position measurements of the drone are transmitted to the central server as $(x,y)$ pairs whenever the drone detects a roundel pattern. Roundels are placed directly adjacent to each cluster of tags as in Fig. \ref{fig:search_algorithm_ground}\subref{fig:tag_search}. The server records each position measurement as a virtual pheromone to be transmitted to any iAnt returning to the central nest.

\begin{figure}
	\centering
	\subfigure[]{
	\includegraphics[scale=1.2]{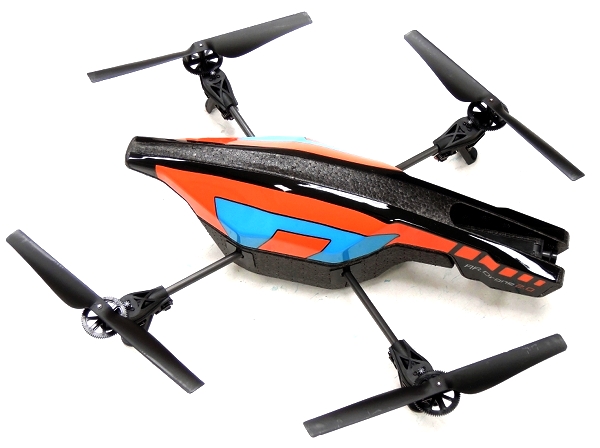}
	\label{fig:drone}
	\hspace{5mm}
	}
	\subfigure[]{
	\includegraphics[scale=1.1]{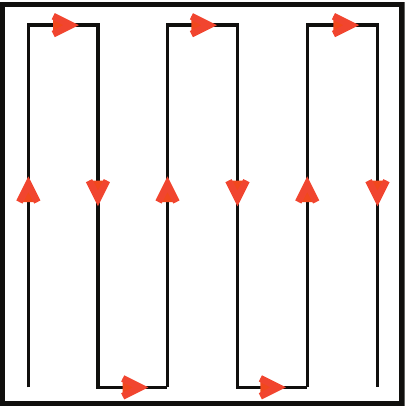}
	\label{fig:complete_search}
	}
	\hspace{5mm}
	\subfigure[]{
	\includegraphics[scale=.2]{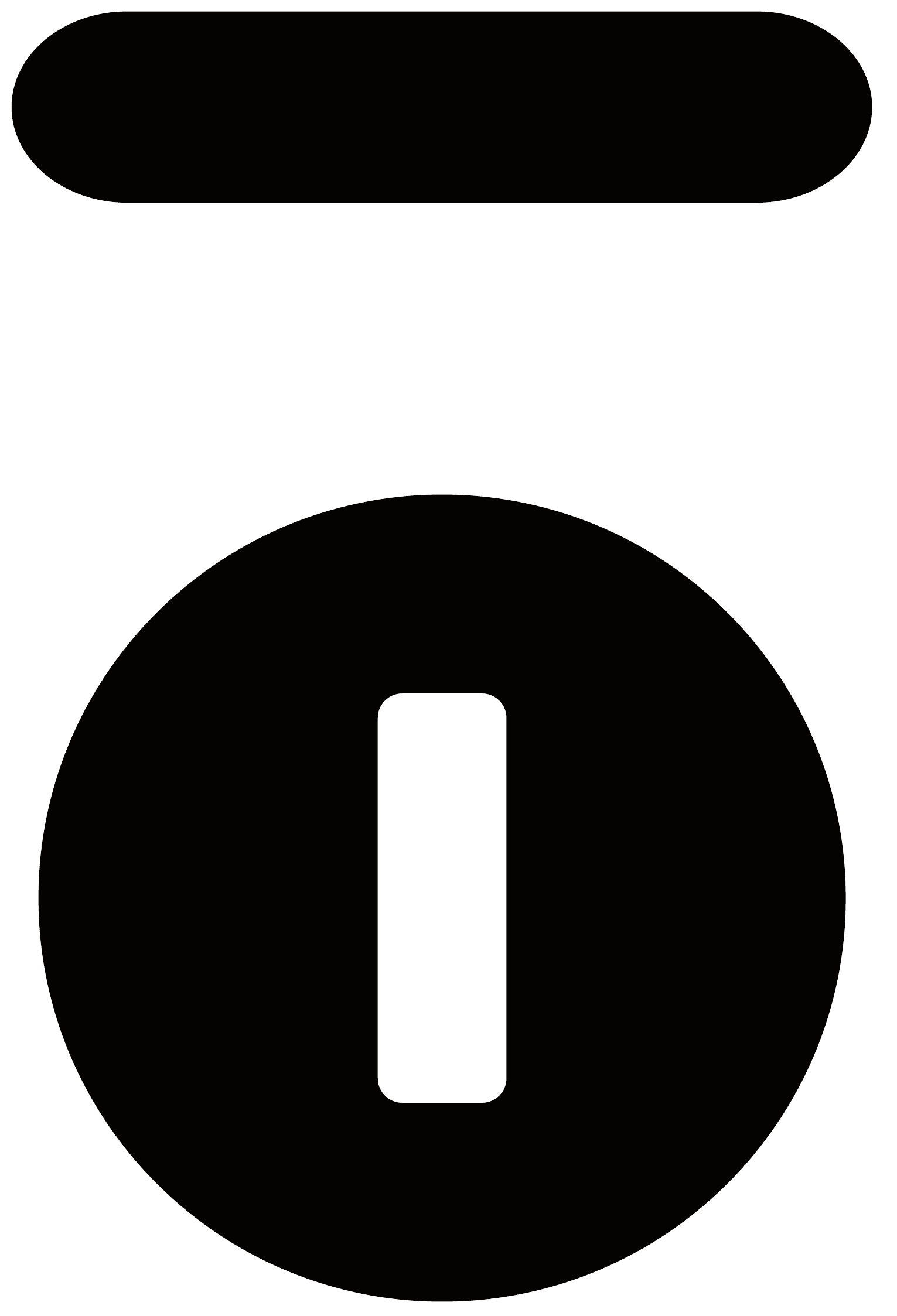}
	\label{fig:roundel}
	}
	\caption{The Parrot AR.Drone \protect\subref{fig:drone} uses a deterministic grid search \protect\subref{fig:complete_search} to explore the experimental area using a built--in object recognition system to detect roundel patterns \protect\subref{fig:roundel}.}
	\label{fig:search_algorithm_aerial}
\end{figure}
 
\subsection{Results}

We present an analysis of the rates at which iAnts retrieve tags from two different distributions, with and without heterogenous coordination with the AR.Drone. Results for each experimental treatment are averaged over three replicants; error bars denote standard error of the mean.

Fig.~\ref{fig:experiment}\subref{fig:live_experiment} shows a snapshot of the one of the experiments conducted at the \marhes Lab with our test bed of heterogeneous robotic systems. A video of the experiment is available at {\em http://swarms.cs.unm.edu/videos.html}.

Fig.~\ref{fig:experiment}\subref{fig:tag_discovery_rate} shows the rate of tag collection per hour of experiment time. In both the single, large cluster and multiple, smaller cluster distributions, we observe that robot teams using heterogeneous coordination produce tag collection rates more than double those without coordination.

\begin{figure}[ht!]
	\centering
	\subfigure[]{
	\raisebox{1mm}{\includegraphics[scale=.78]{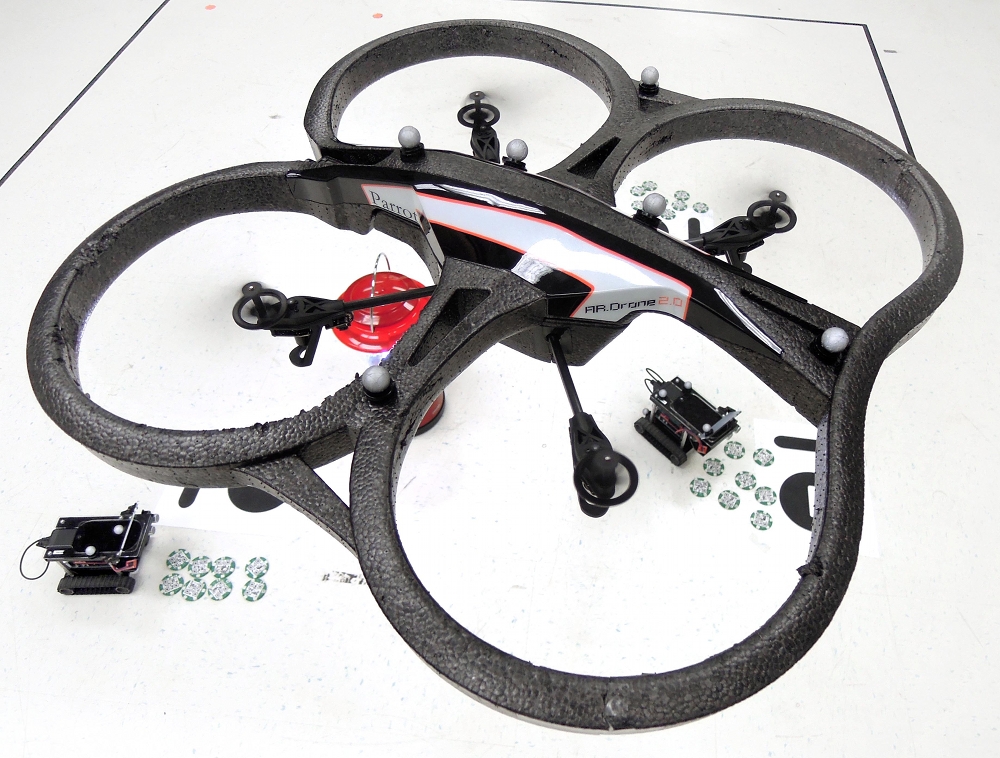}}
	\label{fig:live_experiment}
	}
	\hspace{5mm}
	\subfigure[]{
	\includegraphics[scale=.29]{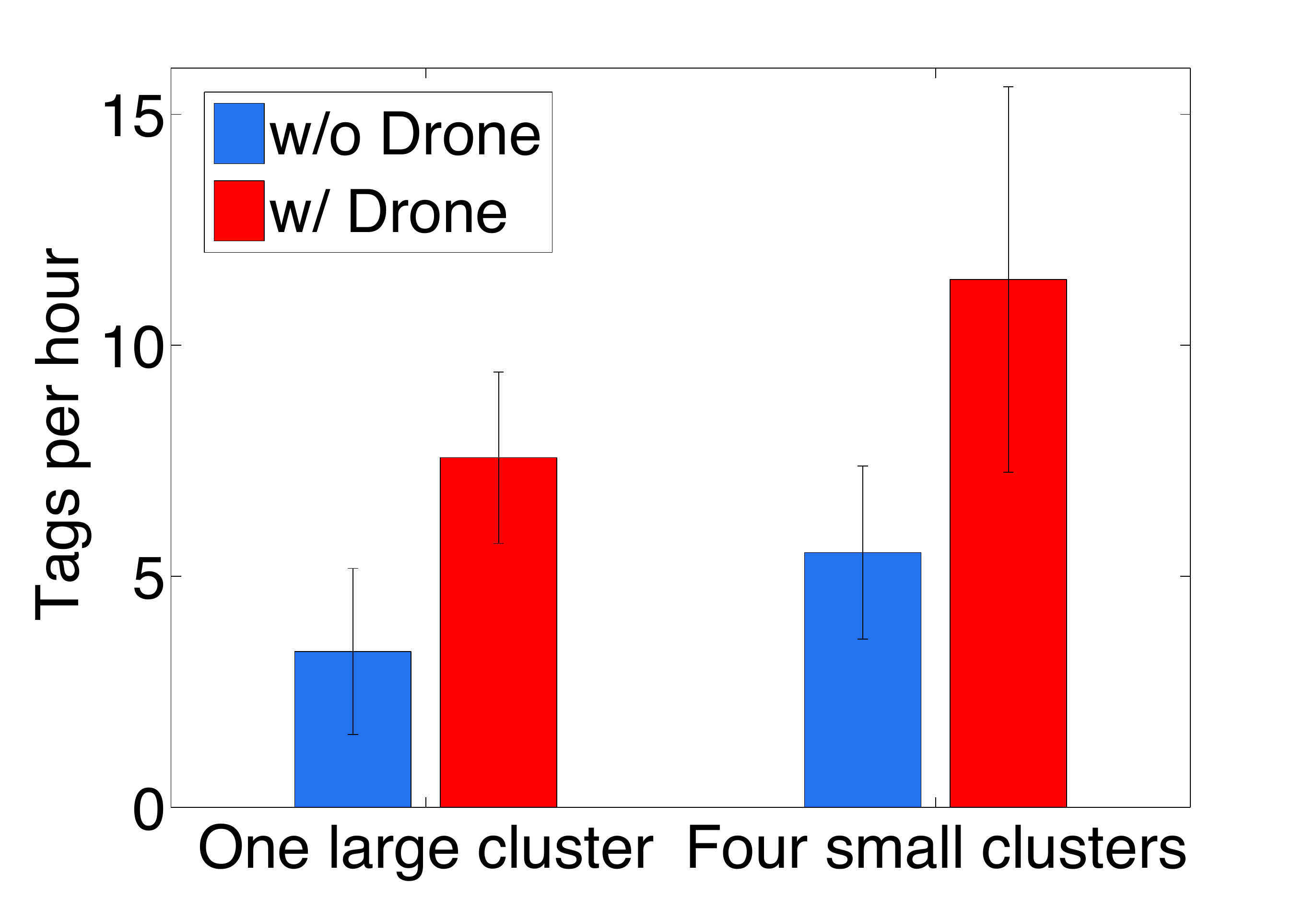}
	\label{fig:tag_discovery_rate}
	}
	\caption{\protect\subref{fig:live_experiment} Physical experiment in progress. \protect\subref{fig:tag_discovery_rate} Rate of tag discovery per hour of experiment time for large and small clusters with and without heterogeneous coordination with the aerial drone. Each bar denotes a mean calculated over three replicants; error bars show standard error of each mean. The video of the experiment is available at {\em http://swarms.cs.unm.edu/videos.html}.}
	\label{fig:experiment}
\end{figure}

\section{Conclusion}\label{sec:5}

In this work, we combined realistic communication, sensing, manipulation, and different dynamical models to improve the performance of a group of autonomous agents. We managed multiple types of robots using our heterogeneous coordination framework to improve the search, coverage, and thus the sensed areas in a workspace.

Besides using realistic sensing and vision capabilities, we have analyzed a communication connectivity algorithm to equilibrate and maintain the network connectivity while exploring the environment, and a power control algorithm to guarantee a certain SINR level between the relay and the sensor and manipulator agents. We presented a simulation scenario in which we implemented a search/pursuit-evasion scenario with a heterogenous network. The robots expand in the environment attracted by potential functions, avoid a large obstacle and reach a partially known target that is manipulated by a specific agent of the system.

We also presented an experiment performed with a heterogeneous system of physical ground and aerial vehicles. We demonstrated that our biologically-inspired ground robots can collect resources twice as fast through coordination with an aerial vehicle. The drone quickly covers and searches a large area  with low-resolution vision sensors, while the slower ground agents perceive smaller but higher resolutions portions of the area. By exploiting the specific characteristics of different agents in our heterogeneous system, we showed that the overall mission performance can be improved.

In future work, we will expand the theoretical results of this paper by introducing more realistic environments with non-convex obstacles and more opponents, and we will evaluate the robustness of our system by simulating failures (\eg some agents stop working).  We will also implement additional experiments with several different types of heterogeneous robotic systems, and test our system on different distributions of resources.
%

\section*{Acknowledgment}
This work was supported by the Micro Autonomous Systems \& Technology (MAST) Program, by NSF grants ECCS \#1027775, IIS \#0812338, and EF \#1038682, by the Department of Energy URPR Grant \#DE-FG52-04NA25590, and by DARPA grant CRASH \#P-1070-113237. 

%


\bibliographystyle{IEEEtran}

\bibliography{nicola_refs_ifac,commotion_refs,nicola_refs2,QuadrotorCommRelayRefs}


%
%
%
%
%


\end{document}